\documentclass[letterpaper]{article} 
\usepackage{aaai24}  
\usepackage{times}  
\usepackage{helvet}  
\usepackage{courier}  
\usepackage[hyphens]{url}  
\usepackage{graphicx} 
\usepackage{natbib}  
\usepackage{caption} 
\usepackage{algorithm}
\usepackage{algorithmic}
\usepackage{booktabs}
\usepackage{multirow}

\usepackage{newfloat}
\usepackage{listings}
\DeclareCaptionStyle{ruled}{labelfont=normalfont,labelsep=colon,strut=off} 
\floatstyle{ruled}
\newfloat{listing}{tb}{lst}{}
\floatname{listing}{Listing}
\usepackage{color}
\usepackage{amsmath}
\usepackage{amssymb}
\usepackage{amsthm}
\usepackage{caption}
\usepackage{subcaption}
\usepackage{mathtools}
\usepackage{appendix}
\usepackage{tabularx,colortbl}

\usepackage{abstract}

\newtheorem{theorem}{Theorem}
\newtheorem{lemma}{Lemma}
\newtheorem{proposition}{Propostion}

\newtheorem{assumption}{Assumption}

\theoremstyle{definition}
\newtheorem{definition}{Definition}

\newtheorem{remark}{Remark}

\newcommand{\kh}[1]{\left(#1\right)}
\newcommand{\sz}[1]{\left|#1\right|}

\renewcommand{\emptyset}{\varnothing}

\def\defeq{\stackrel{\mathrm{def}}{=}}

\def\expec#1#2{{\mathbb{E}}_{#1}\left[ #2 \right]}

\newcommand{\norm}[1]{\left\lVert#1\right\rVert}
\DeclarePairedDelimiter\abs{\lvert}{\rvert}

\newcommand{\calA}{\mathcal{A}}
\newcommand{\calH}{\mathcal{H}}

\newcommand{\calO}{\mathcal{O}}

\newcommand\pp{\boldsymbol{\mathit{p}}}
\newcommand\qq{\boldsymbol{\mathit{q}}}
\newcommand\vv{\boldsymbol{\mathit{v}}}

\renewcommand{\emptyset}{\varnothing}

\DeclareMathOperator*{\argmin}{arg\,min}

\definecolor{lightGreen}{RGB}{150,255,150}

\title{Near-Optimal Resilient Aggregation Rules for Distributed Learning Using 1-Center and 1-Mean Clustering with Outliers}
\author{
    Yuhao Yi\textsuperscript{\rm 1}\equalcontrib, 
    Ronghui You\textsuperscript{\rm 2}\equalcontrib,
    Hong Liu\textsuperscript{\rm 1},
    Changxin Liu\textsuperscript{\rm 3},
    Yuan Wang\textsuperscript{\rm 4}\footnotemark[2],
    Jiancheng Lv\textsuperscript{\rm 1}\thanks{Corresponding author.}
}
\affiliations{
    \textsuperscript{\rm 1}College of Computer Science, Sichuan University\\ 
    \textsuperscript{\rm 2}School of Statistics and Data Science, Nankai University\\
    \textsuperscript{\rm 3}School of Electrical Engineering and Computer Science, KTH Royal Institute of Technology\\
    \textsuperscript{\rm 4}School of Robotics, Hunan University\\
    yuhaoyi@scu.edu.cn, nijuyoumo@gmail.com, changxin@kth.se, yuanw@hnu.edu.cn, lvjiancheng@scu.edu.cn
}
\usepackage{bibentry}

\begin{document}

\maketitle

\begin{abstract}
Byzantine machine learning has garnered considerable attention in light of the unpredictable faults that can occur in large-scale distributed learning systems. The key to secure resilience against Byzantine machines in distributed learning is resilient aggregation mechanisms. Although abundant resilient aggregation rules have been proposed, they are designed in ad-hoc manners, imposing extra barriers on comparing, analyzing, and improving the rules across performance criteria. This paper studies near-optimal aggregation rules using clustering in the presence of outliers. Our outlier-robust clustering approach utilizes geometric properties of the update vectors provided by workers. Our analysis show that constant approximations to the 1-center and 1-mean clustering problems with outliers provide near-optimal resilient aggregators for metric-based criteria, which have been proven to be crucial in the homogeneous and heterogeneous cases respectively. In addition, we discuss two contradicting types of attacks under which no single aggregation rule is guaranteed to improve upon the naive average. Based on the discussion, we propose a two-phase resilient aggregation framework. We run experiments for image classification using a non-convex loss function. The proposed algorithms outperform previously known aggregation rules by a large margin with both homogeneous and heterogeneous data distributions among non-faulty workers. Code and appendix are available at https://github.com/jerry907/AAAI24-RASHB.
\end{abstract}
\section{Introduction}

Distributed machine learning (ML) that involves several collaborative computing machines has been recognized as the backbone for training large-scale ML models in the modern society \cite{warnat2021swarm,MAL-083}. 
However, although distributed ML has significantly improved the efficiency of training process, it tends to be more vulnerable to misbehaving (a.k.a., \textit{Byzantine}) workers. It has been reported in \cite{NEURIPS2019_ec1c5914,KHJ22} that a few Byzantine machines can largely deteriorate the training performance by transmitting falsified information. To this end, Byzantine resilience in distributed ML has recently received increasing attention from both academia and industry.

In resilient distributed ML, a robust distributed optimization algorithm is designed such that the training model remains to be accurate in the presence of a subset of Byzantine workers \cite{FGGPS22}. The key to achieving this objective is a robust aggregation protocol within the server to distill the information sent by the workers. Prior works in this regard are roughly categorized into two classes depending on the property of the training dataset. In the first class, known as the \textit{homogeneous} setup, the data sampled by the workers are assumed to be identically distributed. Another class is the \textit{heterogeneous} setup~\cite{LXCGL19,DD21}, where the data samples among the workers may not precisely reflect the overall population. The difference in datasets induces distinct treatment of Byzantine workers and the best achievable performance \cite{KHJ22}, and the problem of resilient distributed ML under heterogeneity is arguably more challenging \cite{AFGGPS23}.



To facilitate the analysis of the aggregation rules, a line of recent work has established the connections between the properties of aggregators and the performance of the optimization algorithms. 
The paper~\cite{FGGPS22} studied resilient distributed ML with homogeneous data distributions and proposed the concept of $(f,\lambda)$-resilient averaging aggregators. Subsequent work study distributed ML with heterogeneous data distributions, proposing a series of concepts such as the $(\delta_{\max},\zeta)$-agnostic robust~\cite{KHJ22}, the $(f,\kappa)$-robust~\cite{AFGGPS23}, and the $(f, \xi)$-robust averaging~\cite{AGGPS23} aggregators. The relationships of these concepts are also discussed in~\cite{AFGGPS23} and~\cite{AGGPS23}. 

The criteria for the aggregation rules are defined using metrics on subsets of update vectors. However, most of the aggregators are not directly designed to minimize the criteria, resulting in the suboptimality of many aggregation rules. Exceptions include the MDA algorithm in~\cite{FGGPS22} and the SMEA algorithm in~\cite{AGGPS23}, both of which suffer from high computational costs.

The $1$-Center problem, or the minimum enclosing ball problem, is a fundamental problem in computational geometry~\cite{Yil08}, even in low dimensions~\cite{Sma65,Har11}. Its variant with outliers also receives significant interests~\cite{Shy18,Din20}. 1-Mean clustering with outliers is a similar problem using the sum of squared distances as the cost function. Approximation algorithms are intensively studied for clustering problems with outliers~\cite{FKRS19,AISX23,BOR21}. In this paper we propose to use 1-center and 1-mean clustering with outliers as resilient aggregation rules.

\paragraph{Related work.} 
In recent years, Byzantine resilient distributed ML has been intensively studied. Many algorithmic frameworks are proposed~\cite{YCKB18,liu2021approximate,AELA21} to address complex attacks developed progressively~\cite{XKG20,NEURIPS2019_ec1c5914}. These early discussions diverge in assumptions and overall frameworks of the training algorithms. A line of recent work based on resilient aggregating of momentum, or resilient stochastic heavy ball~\cite{KHJ21,KHJ22,FGGPS22,AFGGPS23,AGGPS23}, is the most relevant to its paper. 

Under the framework, properties of some previously designed aggregators are investigated. Some provably optimal aggregators often have high computational costs while many commonly used simple aggregators turn out to be suboptimal. Pre-aggregation steps such as bucketing~\cite{KHJ22} and nearest neighbor mixing (NNM)~\cite{AFGGPS23} are proposed to improve the performance of aggregators.


Clustering algorithms~\cite{GHYR19,SMWS20} are also proposed and analyzed under more restrictive assumptions, where machines in the same cluster have the same data distribution. \citeauthor{EGR18} studied a medoid-based algorithm as an approximation to the geometric median~\cite{CSX17}, but its optimality remains to be studied. 
In this paper, we reduce the problems of optimal aggregation rule
design to 1-center/mean clustering problems with outliers, and apply computationally efficient approximations to construct near-optimal aggregators.


\paragraph{Main contributions.} In this paper, we propose near-optimal aggregators for Byzantine resilient distributed ML using approximation algorithms for the problems of 1-center and 1-mean clustering with outliers. Although the problems of 1-center and 1-mean clustering with outliers are NP-hard, their approximations can be computed efficiently. We show that $2$-approximations achieve near-optimal safety guarantees under existing analytical framework. Specifically, the 1-center with outliers algorithm is optimal for $(f, \lambda)$-resilience and achieves the currently best bound for $(\delta_{\max},\zeta)$-agnostic robustness; the 1-mean with outliers approach is optimal for $(f, \kappa)$-robustness. 

In addition, we discuss two types of contradicting attacks, namely the sneak attack and the siege attack, to show that no single aggregation algorithm, being agnostic about the true distribution of update vectors of normal clients, outperforms other algorithms in all circumstances. To address the dilemma created by the indistinguishability of the two types of attacks, we propose a two-phase aggregation framework. In the framework, 1) the server proposes two candidate sets of parameters using received update vectors; 2) clients elect a set of parameters to commit by evaluating the losses using resampled data. Using clustering-based approaches, the two proposed sets of parameters are easily generated by constructing filters to address the sneak attacks and the siege attacks respectively. 

In summary, this paper 1) proposes near-optimal aggregation rules with provable guarantees by approximating the problems of 1-center/mean clustering with outliers; 2) proposes a two-phase aggregating of optimization framework in which the clustering approaches are used to defend two types of contradicting attacks; 3) empirically shows the advantages of our approach over existing aggregators by performing image classification under various attacks.

\paragraph{Outline.} The remainder of the paper is organized as follows: In Section 2 we introduce the problem setup, some basic concepts and definitions. In Section 3, we introduce the proposed aggregators and prove their robust guarantees. In Section 4, we discuss two types of contradicting attacks which motivates the two-phase aggregating framework. In Section 5 we show empirical results, followed by the section for conclusion and future work.

We use the words robustness and resilience interchangeably in this paper except for formally defined concepts.

\section{Byzantine Resilient Distributed Learning}

In this section, we introduce the Byzantine ML problem, which is followed by a general resilient framework for distributed learning.
Then, we recall some useful robust notions of aggregation rules, under which the distributed learning algorithms are provably resilient and convergent.

\subsection{Problem Setup}

Consider a server-worker distributed learning system with one central server and $n$ workers. Each worker $i\in [n]$ possesses a local dataset consisting of $m$ data points $\mathcal{D}_i:=\{ z_1^{(i)},\dots, z_m^{(i)}  \}$.
The server stores sets of model parameters and update vectors received from the workers. For a given ML model parameterized by $\theta \in\mathbb{R}^d$, each worker $i$ has a local loss function
    $\mathcal{L}_i(\theta) := \frac{1}{m} \sum_{k=1}^m l(\theta, z_k^{(i)})$, 
where $l(\cdot, \cdot)$ represents the loss over a single data point. We assume $l(\cdot, \cdot)$ is differentiable with respect to the first argument, and each $L_i(\cdot)$ is $L$-smooth, that is, $\lVert \nabla \mathcal{L}_i(\theta_1)- \nabla \mathcal{L}_i(\theta_2) \rVert  \leq  L \lVert\theta_1-\theta_2\rVert, \,\, \forall \theta_1,\theta_2\in \mathbb{R}^d$.

We consider a standard adversarial setting where the server is honest and $f$ workers with unknown identities are Byzantine \cite{lamport2019byzantine}.  
The Byzantine workers need not follow the given learning protocol and may behave arbitrarily in the learning process. However they cannot make other workers faulty, falsify the message of any other nodes, or block message passing between the server and any honest (or non-Byzantine) workers.

In real-world ML applications, the datasets held by the honest workers are typically heterogeneous~\cite{shi2023prior}. In this work, we model data heterogeneity by the following standard assumption \cite{KHJ22}.
\begin{assumption}[\bf{Bounded heterogeneity}]\label{assump:bounded_hetero}
Let $\mathcal{H}$ denote the set of indices of honest workers and $\mathcal{L}_{\mathcal{H}}(\theta):=\lvert \mathcal{H}\rvert^{-1} \sum_{i\in\mathcal{H}}\mathcal{L}_i(\theta)$.
    There exists a positive value $G$ such that
        $\frac{1}{\lvert \mathcal{H}\rvert}\sum_{i\in\mathcal{H}} \lVert \nabla \mathcal{L}_i(\theta)- \nabla \mathcal{L}_{\mathcal{H}}(\theta) \rVert^2 \leq G^2, \,\, \forall \theta\in\mathbb{R}^d.$
\end{assumption}

The goal of the server is to approximate a stationary point of $\mathcal{L}_{\mathcal{H}}(\theta)$. Throughout this process, the server iteratively updates a model based on the stochastic gradients received from the workers. 
To proceed, we introduce the concept of Byzantine resilience as follows.

\begin{definition}[\bf{$(f,\varepsilon)$-Byzantine resilience}]
    A learning algorithm is said $(f,\varepsilon)$-Byzantine resilient if, even in the presence of $f$ Byzantine workers, it outputs $\hat{\theta}$ satisfying
        $\lVert \nabla\mathcal{L}_{\mathcal{H}} (\hat{\theta}) \rVert^2 \leq \varepsilon.$
\end{definition}

We note that $(f,\varepsilon)$-Byzantine resilience is generally not possible (for any $\varepsilon$) when $f\geq n/2$ \cite{liu2021approximate}. Therefore, we assume an upper bound for the number of Byzantine workers $f<n/2$ in this work. Furthermore, the heterogeneous datasets render the Byzantine distributed ML much more challenging, as the incorrect gradients from Byzantine workers and the correct gradients from honest workers becomes more difficult to distinguish in this case; see \cite{KHJ22} for a detailed discussion and a lower bound of the training error.

\subsection{Resilient Distributed Learning Algorithm}

We recall a class of resilient distributed learning algorithms in Algorithm~\ref{alg:skeleton}, which we call {Resilient Aggregated Stochastic Heavy Ball (RASHB)} in this work. This approach aggregates stochastic momentum in a resilient manner, which can be applied in both homogeneous~\cite{FGGPS22} and heterogeneous~\cite{AGGPS23} worker settings.


To proceed, we present four useful definitions for the robustness of aggregation rules in the literature, under which the convergence result of Algorithm~\ref{alg:skeleton} follows.

\begin{algorithm}[tb]
\caption{RASHB}
\label{alg:skeleton}
\textbf{Initialization}: for server and worker: Initial model $\theta_0$, initial momentum $m_0^{(1)}=0$, the number of rounds $T$; for each honest worker $w_i$: , robust aggregation $F$, batch size $b$, learning rates $\left\{\gamma_t\right\}_{t=1}^T$, momentum coefficient $\beta$.\\
\begin{algorithmic}[1] 
\FOR{$t= 0, \ldots, T-1$}
\STATE {\bf{Server}} broadcasts $\theta_t$ to all workers;
\FOR{every honest worker $w_i$, $i\in \calH$, in parallel}
\STATE Compute a local stochastic gradient $g_{t}^{(i)}$ using mini-batch data samples;
\STATE Update local momentum:
\begin{equation*}
    m_t^{(i)}=\beta m_t^{(i)} +(1-\beta)g_t^{(i)};
\end{equation*}
\STATE Send $m_t^{(i)}$ to the server;
\ENDFOR
\STATE Server aggregates the received momentums:
$$R_t = F(\{m_t^{(1)}, \ldots, m_t^{(n)}\});$$
\STATE Server updates the model: $\theta_t = \theta_{t-1} - \gamma_t R_t$;
\ENDFOR
\STATE \textbf{return} $\frac{1}{T}\sum_{t=0}^{T-1} \theta_{t}$;
\end{algorithmic}
\end{algorithm}


First, the definition of $(f,\lambda)$-resilient averaging was proposed in~\cite{FGGPS22} for homogeneous data.
\begin{definition}[\bf{$(f,\lambda)$-resilient averaging}]
    Given an integer $f<n/2$ and a real number $\lambda \geq 0$, an 
    aggregation rule $F$ is called $(f,\lambda)$-resilient averaging 
    if for any set of $n$ vectors $X := \left\{x_i\right\}_{i=1}^n$, 
    and any subset $S\subseteq [n]$ with $\sz{S} = n-f$,
    \begin{align}
        \norm{F(X) - \overline{x}_S} \leq \lambda \max_{i,j\in S}\norm{x_i - x_j}\,,
    \end{align}
    where $\overline{x}_S := \frac{1}{\sz{S}}\sum_{i\in S}x_i$.
\end{definition}

Second, to address the heterogeneity in worker's data, \citeauthor{KHJ22} proposed the concept of agnostic robust aggregator (ARAgg) for a (randomized or deterministic) aggregation rule. 
\begin{definition}[\bf{$(\delta_{\max}, \zeta)$-ARAgg}]
    Given a set of $n$ vectors $X:= \left\{x_i\right\}_{i=1}^n$ and a subset $S\subseteq [n]$ with $\sz{S} = n-f$ with $f/n \leq \delta_{\max} < 0.5$ satisfying $\expec{}{\norm{x_i-x_j}^2} \leq \rho^2$ for all $i,j\in S$, 
    the output $F(X)$ of a $(\delta_{\max}, \zeta)$-ARAgg satisfies
    \begin{align}
        \expec{}{\norm{F(X) - \overline{x}_{S}}^2} \leq \zeta \frac{f}{n} \rho^2\,.
    \end{align}
\end{definition}


Third, a stronger notion of $(f,\kappa)$-robustness for aggregation rules was then proposed by the paper~\cite{AFGGPS23}.
\begin{definition}[\bf{$(f,\kappa)$-robustness}]
    Given an integer $f< n/2$ and a real number $\kappa \geq 0$, an aggregation rule $F$ is $(f,\kappa)$-robust if for any set of n vectors $X\:= \left\{x_i\right\}_{i=1}^n$, 
    and any subset $S\subseteq [n]$ with $\sz{S} = n-f$,
    \begin{align}
                \norm{F(X) - \overline{x}_S}^2 \leq \frac{\kappa}{\sz{S}}\sum_{i\in S}\norm{x_i - \overline{x}_S}^2\,.
    \end{align}
\end{definition}

Lastly, a recent work~\cite{AGGPS23} introduced the $(f,\xi)$-robust averaging criterion. Compared with $(f,\kappa)$-robustness, it considers the maximum eigenvalue of the covariance matrix of a set of data points instead of its trace. Therefore it controls the deviation form honest values in \emph{all} directions.

\begin{definition}[\bf{$(f,\xi)$-robust averaging}]
    Given an integer $f< n/2$ and a real number $\xi \geq 0$, an aggregation rule $F$ is $(f,\xi)$-robust averaging if for any set of $n$ vectors $X:= \left\{x_i\right\}_{i=1}^n$, 
    and any subset $S\subseteq [n]$ with $\sz{S} = n-f$,
    \begin{align}
                &\norm{F(X) - \overline{x}_S}^2 \leq \xi\cdot  \lambda_{\max}\kh{M_S}\,,
    \end{align}
    where $\overline{x}_S := \frac{1}{\sz{S}}\sum_{i\in S}x_i$; $\lambda_{\max}(\cdot)$ is the eigenvalue of a matrix; and
    $M_S := \frac{1}{\sz{S}}\sum_{i\in S}(x_i - \overline{x}_S)(x_i - \overline{x}_S)^{\top}$.
\end{definition}

\begin{assumption}[\bf{Bounded variance}]\label{assump:bounded_var}
For each honest worker $i$, there holds that
        $\frac{1}{m}\sum_{z\in\mathcal{D}_i} \lVert \nabla_{\theta} l(\theta,z)-\nabla \mathcal{L}_i(\theta) \rVert^2 \leq \sigma^2, \,\, \forall \theta\in\mathbb{R}^d$.
\end{assumption}

We are in a position to present the convergence results for Algorithm \ref{alg:skeleton}, whose proof can be found in existing works that proposed the robustness definitions. 
\begin{theorem}\label{thm:convergence}
   Suppose Assumptions  \ref{assump:bounded_var} and \ref{assump:bounded_hetero} hold, and recall that $\mathcal{L}_{\mathcal{H}}(\cdot)$ is $L$-smooth. Consider Algorithm \ref{alg:skeleton} and define $\text{Res}_T=  T^{-1}\sum_{t=1}^T\expec{}{\norm{\nabla \mathcal{L}_{\mathcal{H}}(\theta_{t-1})}^2}$.

   \begin{itemize}
       \item[i)] If $F$ is a $(f,\lambda)$-resilient aggregation rule and $G=0$, then $ \text{Res}_T \leq \mathcal{O}\left(\sqrt{(n-f)} \cdot {\lambda} {\sigma}/{\sqrt{T}}\right)$;
        \item[ii)] If $F$ is a $(\delta_{\text{max}},\zeta)$-ARAgg aggregation rule and $G>0$, then
$
           \text{Res}_T \leq \mathcal{O}\left(\zeta {f}G^2/{n}+{\sigma} {\sqrt{\zeta f+1}}/{\sqrt{n T}}\right)
$;
        \item[iii)] If $F$ is a $(f,\kappa)$-robust aggregation rule and $G>0$, then
$
           \text{Res}_T \leq \mathcal{O}\left(\kappa G^2+{\sigma}/{\sqrt{T}}\right)
$;
        \item[iv)] If the aggregation rule $F$ satisfies the condition of $(f,\xi)$-robust averaging and $G>0$, then
$
           \text{Res}_T \leq \mathcal{O}\left(\xi G^2+{\sigma}/{\sqrt{T}}\right)
$.
   \end{itemize}
\end{theorem}
Theorem \ref{thm:convergence} highlights the crucial significance of the resilient aggregation rule in Byzantine distributed learning. This rule not only ensures resilience against Byzantine workers but also influences the overall learning performance.


\section{A Framework for Resilient Aggregation}
In this section we develop a resilient aggregation algorithmic framework using 1-center and 1-mean clustering with outliers. We provide analysis for the proposed aggregation rules to show their near-optimality under various criteria.

\subsection{$1$-Center/Mean Clustering with Outliers}

The \emph{$1$-center clustering} problem is also referred to as the \emph{minimum enclosing ball} problem. The problem is to find a ball with minimum radius containing all given points. The problem with outliers is defined as 
\begin{definition}[\bf{1-center clustering with outliers, or minimum enclosing ball with outliers}]
Given a set of $n$ points $X$ in $\mathbb{R}^d$, and an integer $f<n$ indicating the largest number of faulty points in $X$, find a ball $B(c,r)$ with a center $c\in \mathbb{R}^d$ and a radius $r\in \mathbb{R}$ to cover $(n-f)$ points in $X$, such that $r$ is the minimum among all possible balls.
\end{definition}
The minimum enclosing ball with outliers problem has been shown to be strongly NP-hard~\cite{She13}.

In the \emph{$1$-mean clustering problem}, given a set of $n$ points $X$, the aim is to find a point $c$ in the given space so as to minimize the sum of squared distances from each point $x\in X$ to $c$. The centroid (also called the center of mass) of a set $X$ is defined as $\mathrm{cm}(X):= 1/\sz{X}\cdot \sum_{x_i\in X}x_i$. It is known that the centroid of a given set of points is the optimal solution to the problem. Now we introduce a variant of the problem in the presence of outliers.
\begin{definition}[\bf{1-mean clustering with outliers}]
    Given a set of $n$ points $X$ in $\mathbb{R}$, and the largest number of faulty nodes $f<n$, the problem is to find a vector $c\in \mathbb{R}^d$ so as to minimize 
        $\sum_{x\in K(c)} \norm{c - x}^2$,
    where $K(c)$ is the nearest $(n-f)$ points in $X$ to $c$.
\end{definition}
In a similar manner as the proof of hardness for the problem of minimum enclosing ball with outliers~\cite{She13}, the problem of 1-means clustering with outliers can also be shown to be strongly NP-hard, by a reduction from the $k$-clique problem in regular graphs.


\subsection{Efficient Approximation Algorithms}
The two clustering problems considered are special cases of $k$-center/means clustering with outliers, which are central problems in both geometry and learning. Since both problems are NP-hard, it is natural to seek for approximations. 

\begin{definition}
A polynomial time algorithm $\calA$ is called an $\alpha$-approximation algorithm for a (minimization) optimization problem if for all instances of the problem, $\calA$ produces a solution whose value is guaranteed to be at most $\alpha$ times the optimum value.    
\end{definition}

In this paper we use simple approximation algorithms that consider all data points in $X$ as cluster center candidates instead of all points in $\mathbb{R}^d$. Mostly in the $k$-means setup, cluster centers chosen from data points are also called medoids. Given a cluster center $x$ and a set of data points $S$, we define the cost function
\begin{align}
\label{eqn:costdefine}
    \mathrm{cost}(c, S) := 
    \begin{cases}
        \max_{x\in S}\norm{x - c}  & {\text{(CenterwO)}}\\
        \sum_{x\in S}\norm{x - c}^2  & {\text{(MeanwO)}}
    \end{cases}
\end{align}
for the approximate 1-center/mean clustering problem. The aggregation rules using approximations to the $1$-center/mean clustering with outliers are shown in Algorithm~\ref{alg:ClusterwO}.


\begin{algorithm}[tb]
\caption{CenterwO/MeanwO}
\label{alg:ClusterwO}
\textbf{Input}: a set of $n$ vectors $X$ in $\mathbb{R}^d$, and an integer $f < \frac{n}{2}$.\\
\textbf{Output}: the mass center of the $n-f$ points in an approximate minimum 1-center/mean cluster with $f$ outliers.
\begin{algorithmic}[1] 
\FOR{ $i = 1,\ldots, n$}
\STATE Find $K(x_i)$, the $n-f$ closest vectors in $X$ to the vector $x_i$ (including $x_i$), breaking ties arbitrarily;
\STATE Let ${\mathrm{dist}}_i = \mathrm{cost}(x_i, K(x_i))$;\\
 {\text{/*see (\ref{eqn:costdefine}) for definition of 1-center/mean cost*/}}
\ENDFOR
\STATE Let $j \in \argmin_{i} {\mathrm{dist}}_i$;
\RETURN $\frac{1}{n-f}\sum_{x\in K(x_j)} x$;
\end{algorithmic}
\end{algorithm}


\begin{lemma}
\label{lem:approxRatio}
    The {\rm{CenterwO}} and {\rm{MeanwO}} algorithms are $2$-approximations to the problems of 1-center and 1-mean clustering with outliers, respectively.
\end{lemma}

The CenterwO algorithm was mentioned in many previous work, a proof for its approximation guarantee was given in~\cite{She13}. We prove the $2$-approximation of MeanwO in the technical appendix. The running time of the CenterwO/MeanwO algorithm is $\calO(n^2 d)$, and the memory usage is $\calO(nd)$.

We note that the approximation algorithms for the considered problems can be improved in approximation ratio, running time, and memory usage. Examples of such improvements for the one center with outliers problem include a deterministic $\calO(1)$-approximation algroithm with $\calO(nd)$ running time~\cite{Shy18} and a deterministic $\calO(1)$-approximation streaming algorithm using $\calO(fd)$ memory~\cite{MK08}. Nevertheless, we adopt Algorithm~\ref{alg:ClusterwO} in our analysis and experiments.




\subsection{Analysis of the Algorithms}

The robust properties of the CenterwO and MeanwO aggregation rules are summarized in Theorem~\ref{thm:Clustering}, the proof of which is shown in the technical appendix.

\begin{theorem}
\label{thm:Clustering}
    For any $f< n/2$, $f/n\leq \delta_{\max} <1/2$, and $\nu \defeq 1/2-\delta_{\max}$, the {\rm{CenterwO}} algorithm is 
    \begin{itemize}
    \item $\kh{f, \frac{(2\sqrt{2}+1)f}{n-f}}$-resilient averaging, 
    \item $\kh{\delta_{\max}, \frac{(18+8\sqrt{2})}{{(1+2\nu})^2} \frac{f}{n}}$-agnostic robust, 
    \item $\kh{f, \frac{8f^2+2f}{n-2f}\cdot \frac{n-f}{n-2f} }$-robust,
    \item $\kh{f, \frac{8f^2+2f}{n-2f}\cdot \frac{n-f}{n-2f} }$-robust averaging;
    \end{itemize}
    the {\rm{MeanwO}} algorithm is 
    \begin{itemize}
    \item $\kh{f, \frac{\sqrt{3f(n-f)}}{(n-2f)}}$-resilient averaging,
    \item $\kh{\delta_{\max}, \frac{3(1+2\nu)}{8\nu^2}}$-agnostic robust,
    \item $\kh{f, \frac{6f}{n-2f} \cdot \frac{n-f}{n-2f} }$-robust,
    \item $\kh{f, \frac{6f\cdot \min\{n-f, d\}}{n-2f} \cdot \frac{n-f}{n-2f} }$-robust averaging.
    \end{itemize}
\end{theorem}

\begin{remark}
    The $(f, \lambda)$-resilient aggregating propoerty of CenterwO matches the lower bound 
    $\lambda \geq f/(n-f)$ 
    shown in~\cite{FGGPS22}, up to a constant factor. 
    The only previously known aggregator that matches this bound is MDA~\cite{EGR18}, which runs in $\calO(\genfrac(){0pt}{2}{n}{f}+ n^2d)$ time. The CenterwO algorithm is much more efficient with moderate $n$ and $f$. 
\end{remark}

\begin{remark}
    CenterwO algorithm is the first known $(\delta_{\max}, \calO(f/n))$-agnostic robust aggregator. MeanwO is a $(\delta_{\max}, \calO(1))$-agnostic robust aggregator, which matches the previously known best $\calO(1)$ results achieved by CClip (not agnostic),  Bucketing+Krum, and Bucketing+GM~\cite{KHJ22}.
\end{remark}

\begin{remark}
    The $(f, \kappa)$-robust aggregating propoerty of CenterwO matches the lower bound 
    $\kappa \geq f/(n-2f)$ 
    shown in~\cite{AFGGPS23} up to a constant factor, given that $\nu$ is a constant. The CWTM algorithm matches this bound without pre-aggregation steps, so does the SMEA algorithm~\cite{AGGPS23}. With the mixing pre-aggregation step, NNM+(GM/CWMed/CWTM/Krum)~\cite{AFGGPS23} matches this bound.
\end{remark}

Although CWTM also achieves near-optimal $(f,\kappa)$-robustness, it removes $2f$ values in each coordinate, and therefore utilizes less information in the data, making the training algorithm potentially inefficient. Pre-aggregation steps can be applied prior to any aggregation rule, at the cost of additional processing time at the server end. SMEA and other eigenvector-based algorithms~\cite{AGGPS23,DD21,ZWPWJ+23} are generally more expensive, with $\Omega(\genfrac(){0pt}{2}{n}{f}n^2d)$ or $\Omega(n d^2)$ running time. 


\section{The Two-Phase Aggregation Framework}
In this section, we propose a 2-phase aggregation framework motivated by a dilemma arising from a thought experiment.
\subsection{Two Contradicting Types of Attacks}
The Byzantine clients and the aggregation rule play a zero sum game to decide the model update in each round of the learning process. We provide a simple example to show that no fixed aggregation rule outperforms other strategies under \emph{all} attacker strategies.

\begin{figure}
     \centering
     \begin{subfigure}{0.2\textwidth}
         \includegraphics[width=\textwidth]{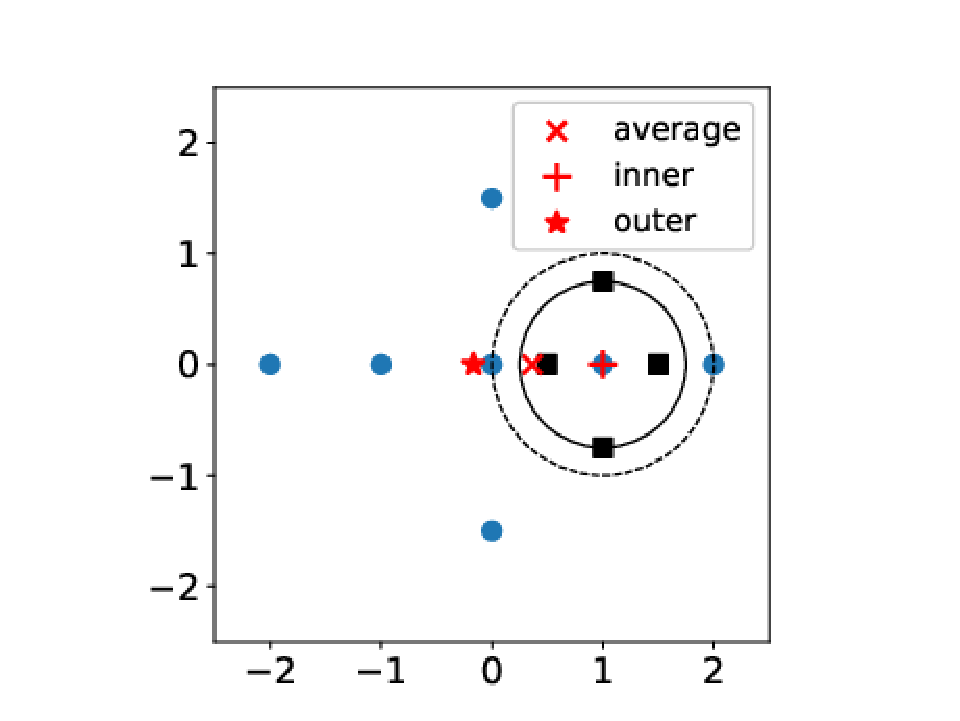}
         \caption{sneak attack}
         \label{fig:toyInAttack}
     \end{subfigure}\!\!
     \begin{subfigure}{0.2\textwidth}
         \includegraphics[width=\textwidth]{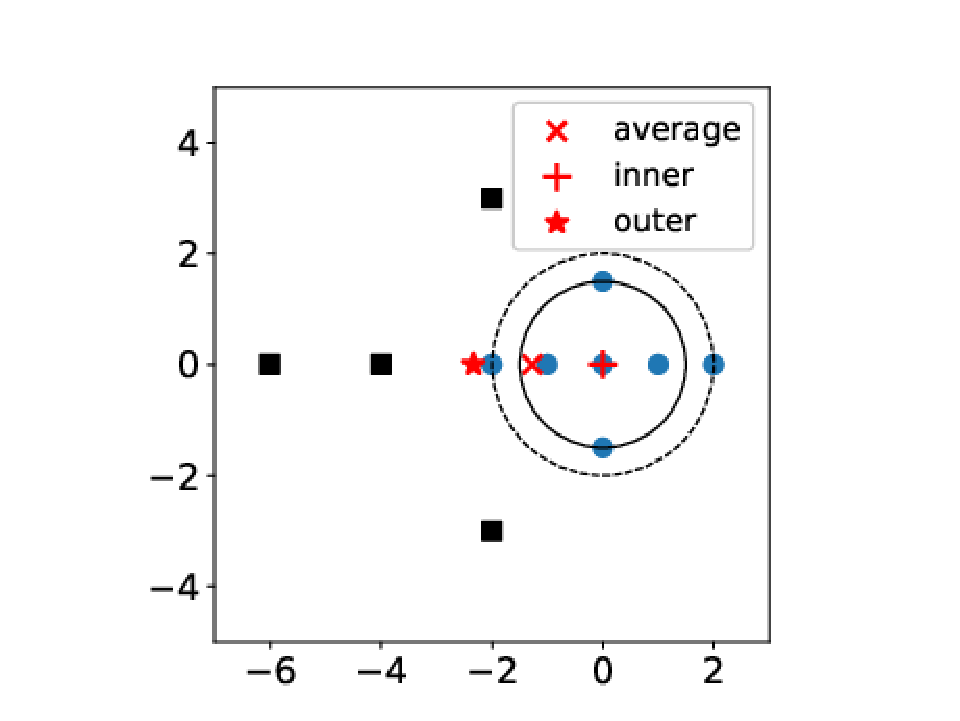}
         \caption{siege attack}
         \label{fig:toyOutAttack}
     \end{subfigure}
 \caption{Two types of attacks produce a dilemma for any robust aggregation rule. Blue circles are update vectors produced by honest clients. Black squares show update vectors provided by Byzantine clients. Aggregated values of the naive averaging rule and 1-center/mean with outliers rules (inner and outer averaging) are shown in~(\ref{fig:toyInAttack}) and (\ref{fig:toyOutAttack}).}
\vspace{-4.5mm}
\end{figure}

Suppose $n=11$ and $f=5$, and there are indeed $7$ honest clients and $4$ Byzantine clients. In Figure (\ref{fig:toyInAttack}) and Figure (\ref{fig:toyOutAttack}) the blue circles show $7$ update vectors produced by honest clients. We discuss two types of attacks to shift the average vector from the true mean $(0,0)$. In Figure (\ref{fig:toyInAttack}), the attacker shift the average by placing $4$ biased update vectors within the convex hull of the blue points. In Figure (\ref{fig:toyOutAttack}), the attacker place $4$ biased update vectors around the update vectors of honest clients. We call the first type of attack the \emph{sneak attack} and the second type the \emph{siege attack}. 

If we remove all labels of the vectors, the two cases are \emph{indistinguishable} for \emph{all} aggregation algorithms without knowing the range of honest update values or their variance. Furthermore, any algorithm outperforms the naive average algorithm in the first case produces larger bias than the average algorithm in the second case, and vice versa. Therefore no aggregation rule dominates other rules under all attacks.

Although Figure (\ref{fig:toyInAttack}) and (\ref{fig:toyOutAttack}) are artificially constructed examples, we claim that the issue exists in real update vector filtering problems. Real attacks can roughly be categorized into these two types, depending on whether the Byzantine vectors are placed within or out of the distribution of the update vectors of honest clients.

Most of the existing methods focus on preventing the siege attack by reducing the impact of outliers. This is because the siege attack can add unbounded bias to the model update, while the bias added by the sneak attack is restricted by the variance of the honest client updates. 
However, when the data distribution among clients are heterogeneous, the honest updates are not identically distributed. 
 In this case the guarantees given by the optimization algorithms heavily depend on the smoothness of the loss function, which may not hold in deep models used in practice.

Given the information of the type of attack used, we can use 1-center/means clustering with outliers algorithm to remove the faulty update vectors effectively. In Figure~(\ref{fig:toyInAttack}), an outer-cluster averaging rule identifies an optimal cluster of $f$ points, and returns the average of the vectors \emph{outside} the cluster. In Figure~(\ref{fig:toyOutAttack}), an inner-cluster averaging rule identifies an optimal cluster of $(n-f)$ points, and returns the average of the vectors \emph{inside} the cluster. As the figures show, if both aggregation rules are used, one of the rules returns a value close to the true average. It remains a problem to choose one from the two aggregated values.

\subsection{The 2PRASHB Algorithm}
We propose a two-phase aggregation framework to let the clients elect a model from two candidate models in each round. We note that if we let each client to directly adopt one model from the two proposed ones, it will result in multiple models in the system\footnote{which is called a split brain state in distributed systems.}. 
An instantiation of the framework is given in Algorithm~\ref{alg:skeleton2p}. The algorithm is called the two-phase aggregation stochastic heavy ball algorithm (2PRASHB). In the first phase, the \emph{prepare} phase, 1) the server broadcasts the current model; 2) the workers send their local updates to the server; 3) the server then proposes two models using the Inner and Outer aggregation algorithm\footnote{See appendix for details.}. 
In the second phase, the \emph{voting} phase, 1) each honest client then samples a new batch of data to evaluate the loss of the two proposed models, and the honest clients then send their votes to the server\footnote{Byzantine workers can send any message or send nothing.}; 2) the server acknowledges the winning model. 

We note that the server is not allowed to propose more than two candidate models for the clients to choose. This is because if there are more than two choices, vote splitting can happen between models with similar losses, and the Byzantine clients may win the election with $f<n/2$.

\begin{algorithm}[tb]
\caption{2PRASHB}
\label{alg:skeleton2p}
\textbf{Initialization}: The same as Algorithm~\ref{alg:skeleton}.\par
\begin{algorithmic}[1] 
\FOR{$t= 0, \ldots, T-1$}
\STATE {\bf{Server}} broadcasts $\theta_t$ to all workers;
\STATE Clients calculate and send momentums $\{m_t^{(i)}\}_{i=1}^n$ by executing line 2-7 in Algorithm~\ref{alg:skeleton};
\STATE Server calculates two values from the received momentums:
\begin{align}
R_t &= \mathrm{Inner}(\{m_t^{(1)}, \ldots, m_t^{(n)}\}),\\
Q_t &= \mathrm{Outer}(\{m_t^{(1)}, \ldots, m_t^{(n)}\}).
\end{align}
\STATE Server proposes two updated models with parameters: $\widetilde{\theta}_t = \theta_{t-1} - \gamma_t R_t$ and $\widehat{\theta}_{t} = \theta_{t-1} - \gamma_t Q_t$, and send both $\widetilde{\theta}_t$ and $\widehat{\theta}_t$ to all clients;
\FOR{every honest worker $w_i$, $i\in \calH$, in parallel}
\STATE Evaluate the loss of the two proposed models $\widetilde{\theta}_t$ and $\widehat{\theta}_t$ on a new mini-batch of data samples;
\STATE Choose one set of parameters with smaller loss, and send its choice to the server;
\ENDFOR
\STATE Server chooses the model which wins the popular vote (breaking ties arbitrarily), and sets it as  $\theta_t$;
\ENDFOR
\STATE \textbf{return} $\frac{1}{T}\sum_{t=0}^{T-1} \theta_{t}$;
\end{algorithmic}
\end{algorithm}

\section{Experiments}



In this section, we run simulations for an image classification task with a non-convex objective. We run four algorithms based on our study: Cent1P and Mean1P adopt the RASHB framework, using the CenterwO and MeanwO algorithm as the aggregator; Cent2P and Mean2P adopt the 2PRASHB framework. The Inner and Outer aggregators are instantiated by 1-center and 1-mean clustering with outliers. Limited by space, we leave more details about the experiments to the appendix.


\paragraph{Generation of datasets.} The FEMNIST dataset, a standard benchmark for distributed and federated learning, is constructed by partitioning data in the EMNIST~\cite{CATV17} dataset. We sample $5\%$ of the images in the original dataset to construct our datasets. The FEMNIST dataset has 805,263 images under 62 classes.For the homogeneous setting, each client sample images from a uniform distribution over 62 classes. We generate heterogeneous datasets for clients using categorical distributions $\qq$ drawn from a Dirichlet distribution $\qq \sim \mathtt{Dir}(\alpha \pp)$, where $\pp$ is a prior class distribution over 62 classes~\cite{HQB19}. Each client sample from a categorical distribution characterized by an independent $\qq$. In our experiment for the heterogeneous setting, we let $\alpha = 0.1$, which is described as the \emph{extreme heterogeneity} setting in~\cite{AFGGPS23}. For each worker, the training set and testing set are sampled independently from a distribution characterized by the same $\qq$. Due to space limitations, similar results for CIFAR10~\cite{Krizhevsky2009LearningML} are deferred to appendix. 

\paragraph{Adversarial attacks.} 
We run experiments for $3$ levels of adversarial rates: $0.1$, $0.2$, and $0.4$, i.e.\ $3$, $7$, and $14$ out of $35$ clients are corrupted. The honest workers are always honest during the learning process. The Byzantine workers send corrupted update vectors to the server, and vote for the model with larger loss in the voting phase if the two-phase framework is applied. We simulate $6$ commonly studied adversarial attacks: the label flipping attack \textbf{LF}, the sign flipping attack \textbf{SF}, the random Gaussian attack \textbf{Gauss}, the omniscient attack \textbf{Omn}~\cite{NIPS2017_f4b9ec30}, the fall of empire attack \textbf{Empire} ~\cite{XKG20}, and the scaled variance attack \textbf{SV}~\cite{NEURIPS2019_ec1c5914, AELA21}.We also customize a more sophisticated attack tailored to aggregation rules, PGA algorithm~\cite{SHKR22}, to attack various algorithms. See appendix for complementary description about all attack algorithms.


\paragraph{Baselines.}  We compare the proposed algorithms with $6$ baseline aggregation rules: the naive average (Avg), Geometric Median (GM) approximated by the Weiszfeld's algorithm~\cite{PKH22} with the $1$ iteration, Centered Clipping (CClip)~\cite{KHJ21} with hyperparameters $\vv =\mathbf{0}$ and $\tau=0.215771$, Coordinate-Wise Median (CWM)~\cite{YCKB18}, Coordiante-Wise Trimmed Mean (CWTM)~\cite{YCKB18}, and Krum~\cite{NIPS2017_f4b9ec30}. We apply all the baseline aggregators to the RASHB framework.

\paragraph{Architecture and hyperparameters.} 
In our study, we employ a Convolutional Neural Network (CNN) comprising two convolutional layers (see appendix for details). 
with a learning rate of 0.1 and momentum of 0. The training process is carried out over 1500 rounds with a batch size of 3. We run all models with different aggregation rules under each attack five times, each with different random seeds. Finally, we report the averages of performance across these runs\footnote{Standard deviations across runs are shown in the appendix.}. 
The implementation is based on RFA\footnote{\url{https://github.com/krishnap25/tRFA}} and MEBwO\footnote{\url{https://github.com/tomholmes19/Minimum-Enclosing-Balls-with-Outliers}}.

\paragraph{Experimental results.} 
Table \ref{tab:iid} shows the performance of different aggregation algorithms and adversarial attacks on the uniform sampling dataset at an adversarial rate of $0.4$. The results demonstrate the robustness of Cent2P and Mean2P, as they consistently achieved the highest worst performance across different attack scenarios. Specifically, Cent2P achieved a minimum accuracy of 0.62, while Mean2P achieved 0.61, both outperforming the third method, CClip, with a minimum accuracy of 0.20. Furthermore, only Cent2P and Mean2P achieved an accuracy above 0.75, whereas the accuracy of all the other methods remained below 0.20 and 0.32 under the Omn and SV attack scenarios respectively. These results underscore the remarkable effectiveness of Cent2P and Mean2P in maintaining strong performance, even when confronted with challenging adversarial conditions.

\begin{table}[ht]
    \centering
    \caption{Performance comparison on the uniform sampling datasets at an adversarial rate of $0.4$.}
    \label{tab:iid}
    \scalebox{0.75}{
    \begin{tabular}{ccccccccccccccccccccc}
        \toprule
        Aggregation & LF & SF & Gauss & Omn & Empire & SV & Worst \\
        \midrule
        Avg & 0.56 & 0.04 & 0.79 \cellcolor{lightGreen}& 0.00 & 0.79 \cellcolor{lightGreen}& 0.32 & 0.00 \\
        GM & 0.64 & 0.46  & 0.79 \cellcolor{lightGreen}& 0.00 & 0.74 & 0.07 & 0.00\\
        CClip & 0.58 & 0.45 & 0.64 & 0.20 & 0.26 & 0.26 & 0.20\\
        CWM & 0.50 & 0.45 & 0.54 & 0.02 & 0.07 & 0.05 & 0.02\\
        CWTM & 0.51 & 0.39 & 0.55 & 0.02 & 0.05 & 0.05 & 0.02\\
        Krum & 0.53 & 0.36 & 0.47 & 0.10 & 0.01 & 0.05 & 0.01\\
        Cent2P & 0.74 \cellcolor{lightGreen}& 0.62 \cellcolor{lightGreen}& 0.76 & 0.79 \cellcolor{lightGreen}& 0.74 & 0.76 \cellcolor{lightGreen}& 0.62\cellcolor{lightGreen}\\
        Mean2P & 0.73 & 0.61 & 0.76 & 0.79 \cellcolor{lightGreen}& 0.73 & 0.75 & 0.61\\
        \bottomrule
    \end{tabular}
    }
\end{table}
\vspace{-2mm}
\begin{figure}[ht]
\setlength{\abovecaptionskip}{2mm}
\vspace{-2mm}
     \centering
\includegraphics[width=0.45\textwidth]{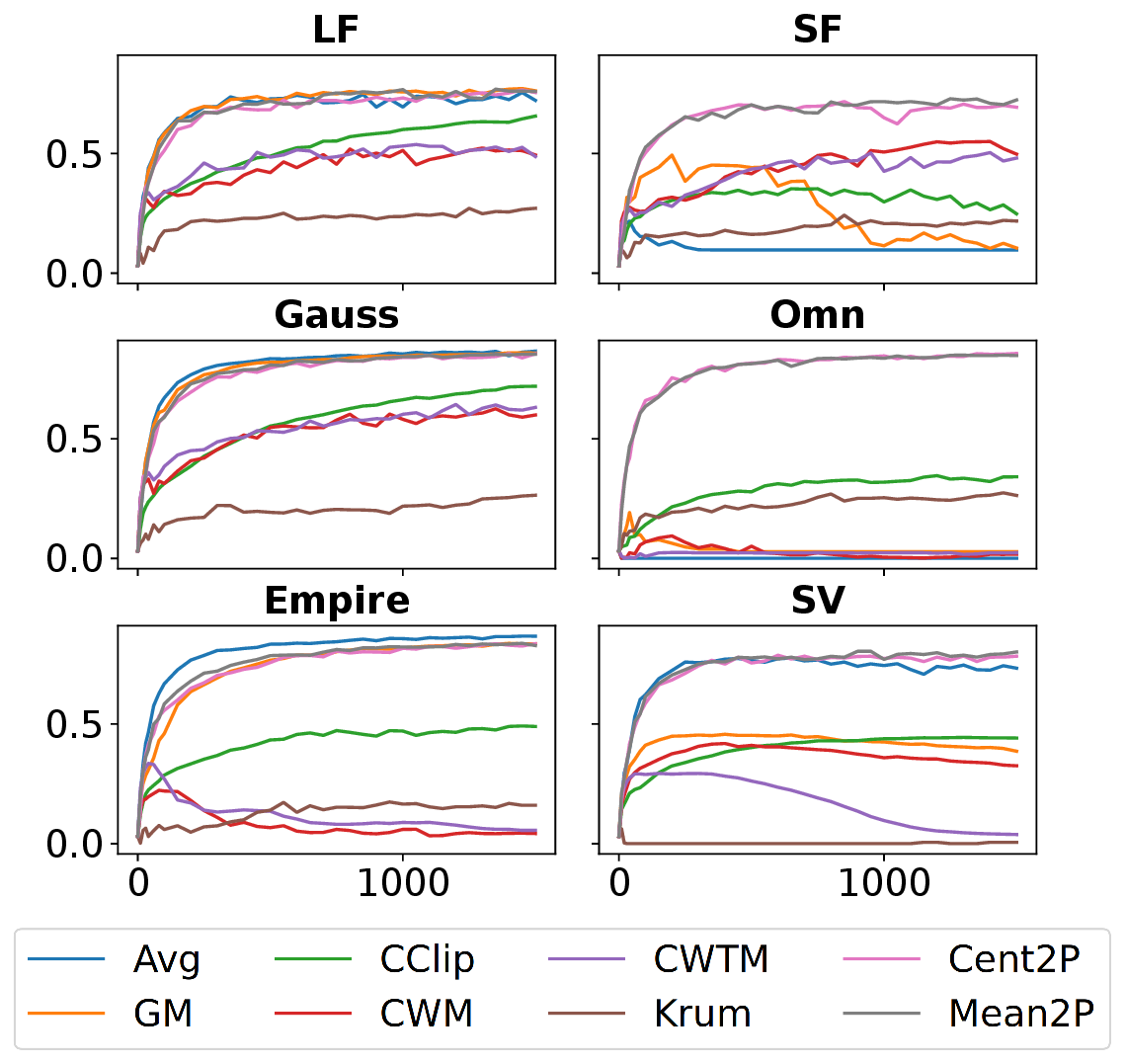}
 \caption{Performance comparison on the heterogeneous datasets at an adversarial rate of $0.2$, with the x, y axis representing testing accuracy and step number, respectively.}
 \label{fig:niidResult}
\vspace{-4mm}
\end{figure}
Figure~\ref{fig:niidResult} illustrates the testing accuracy on the heterogeneous datasets at an adversarial rate of $0.2$. The results highlight that when subjected to the SF and Omn attacks, Cent2P and Mean2P exhibit significant advantages over all other aggregation rules. When facing the SV attack, Cent2P and Mean2P display a slight advantage over Avg and much outperform the remaining baseline methods. Conversely, Avg achieves the best performance under the Empire attack, while certain resilient aggregation rules experience significant degradation. Cent2P, Mean2P, and GM yield comparable results to Avg. In the presence of LF and Gauss attacks, Cent2P, Mean2P, Avg, and GM stand out as having the highest testing accuracy.

\section{Conclusion and Future Work}

We have proposed a near-optimal resilient aggregation framework based on  1-center and 1-mean clustering with outliers. Approximation algorithms have been applied to achieve both efficiency and resilience/robustness. We have proven safety guarantees provided by the proposed algorithms. We have proposed a two-phase resilient aggregation framework based on the observation that no single aggregation rule outperforms other rules against two contradicting types of attacks. We have shown the advantages of the proposed approaches by running numerical simulations for image classification with non-convex loss in the homogeneous and heterogeneous settings. 
Future work may study the resilience of other outlier-robust clustering methods, and the theoretical guarantee of the 2PRASHB framework. 
\section*{Acknowledgments}

Yuhao Yi and Hong Liu are supported by the National Natural Science Fundation of China (Grant No. 62303338) and the Fundamental Research Funds for the Central Universities, under grant Sichuan University YJ202285. This work is supported by the Key Program of National Science Foundation of China (Grant No. 61836006) and Engineering Research Center of Machine Learning and Industry Intelligence (Ministry of Education), Chengdu 610065, P. R. China. This research is also supported by an NSERC Postdoctoral Fellowship.
\clearpage


\clearpage

\onecolumn
\section*{Organization of the Appendix}
Appendix~\ref{appdxSec:proofApprox} provides a proof of the 2-approximation of the MeanwO algroithm. Appendix~\ref{appdxSec:proofMain} contains a proof of Theorem 2, which provides robust guarantees for CenterwO and MeanwO. Appendix~\ref{appdxSec:twophase} explains details for the Inner and Outer algorithm in the 2PRASHB framework. Appendix~\ref{appdxSec:experiments} contains details about the running environment of the experiments, description of the adversarial attacks, all experimental results, and an ablation study of the two-phase framework.
\appendix
\section{Proof of the 2-Approximation of MeanwO}
\label{appdxSec:proofApprox}
We recall the following proposition.
\begin{proposition}[Propostion 3 in~\cite{BOR21}]
\label{prop:madoidMean}
Let $X$ be a finite set of points in $\mathbb{R}^d$, then
\begin{align}
    \min_{y\in X} \sum_{x\in X} \norm{x-y}^2 \leq 2 \sum_{x\in X}\norm{x-\mathrm{cm}(X)}^2\,.
\end{align}
\end{proposition}

\begin{proof}[Proof of Lemma~1]
    We recall the notation $K(c)$, which represents the set of $(n-f)$ data points in $X$ closest to the point $c$ (breaking ties arbitrarily). Let $\widehat{c}$ be the optimal cluster center for the problem of 1-mean clustering with outliers, and $\widehat{S}:=\{i \mid x_i\in K(\widehat{c})\}$. We denote by
    \begin{align*}
        j\in  \argmin_{j\in \widehat{S}} \sum_{i\in\widehat{S}}\norm{x_i-x_j}^2\,,
    \end{align*}
    then by Proposition~\ref{prop:madoidMean},
    \begin{align}
    \label{eqn:twoApxMean}
    \sum_{i\in \widehat{S}}\norm{x_i-x_j}^2 \leq 2 \sum_{i\in \widehat{S}}\norm{x_i-\overline{x}_{\widehat{S}}}^2\,.
    \end{align}
    By the definition of $K(x_j)$ we attain  
    \begin{align}
    \label{eqn:NearestNbr}
        \sum_{i\in K(x_j)}\norm{x_i-x_j}^2 \leq \sum_{i\in\widehat{S}}\norm{x_i - x_j}^2\,.
    \end{align}
    The MeanwO algorithm chooses a data point $x_k$ as the cluster center (medoid) such that
    \begin{align}
    \label{eqn:MeanwOOutput}
        x_k\in \argmin_{x_i\in X}\sum_{\ell\in K(x_j)}\norm{x_\ell - x_i}^2\,.
    \end{align}
    Then we attain
    \begin{align}
        &\sum_{i\in K(x_k)}\norm{x_i-x_k}^2 \nonumber\\
        &\leq \sum_{i\in K(x_j)}\norm{x_i-x_j}^2 \qquad\qquad&\text{by applying (\ref{eqn:MeanwOOutput})}\nonumber\\
        &\leq \sum_{i\in\widehat{S}}\norm{x_i - x_j}^2 \qquad\qquad&\text{by applying (\ref{eqn:NearestNbr})}\nonumber\\
        &\leq 2 \sum_{i\in \widehat{S}}\norm{x_i-\overline{x}_{\widehat{S}}}^2\,.
        \label{eqn:approximation}
    \end{align}
    The last inequality is due to (\ref{eqn:twoApxMean}).
    Note that (\ref{eqn:approximation}) is exactly the definition of $2$-approximation.
\end{proof}

\section{Proof of Robust Properties of CenterwO and MeanwO}
\label{appdxSec:proofMain}
\begin{proof}[Proof of Theorem 2]
    Given a set of $n$ vectors $X$ in $\mathbb{R}^d$,
    let $B(c^*, r^*)$ be an optimal solution to the problem of 1-center clustering with outliers. The set of vectors covered by $B(c^*, r^*)$ is denoted as $S^*$. The {\rm{CenterwO}} finds a ball $B(c',r')$ which contains $n-f$ vectors in $X$ satisfying $r'\leq 2 r^*$. We denote by $S'$ the set of point indices based on which the returned vector is calculated, i.e.\ $S':=\{i \mid x_i\in K(c')\}$. For any subset $S\subseteq[n]$ we denote by $B(c_S, r_S)$ the minimum enclosing ball of the set of vectors $\{x_i\}_{i\in S}$. Then we attain 
    \begin{align}
        &\norm{{\rm{CenterwO}}(X)-\overline{x}_{S}}\nonumber\\
        & = \norm{\frac{1}{n-f}\sum_{i\in S'} x_i - \frac{1}{n-f}\sum_{i\in S} x_i}\nonumber\\
        & = \norm{\frac{1}{n-f}\sum_{i\in S'\backslash S} x_i - \frac{1}{n-f}\sum_{i\in S\backslash S'} x_i}\nonumber\\
        & \leq \frac{f}{n-f} \cdot \max_{i\in S\backslash S',j\in S' \backslash S} \norm{x_i-x_j}\,.\label{eqn:CenterDeviation}
    \end{align}
    The inequality follows by the fact that $\sz{S\backslash S'} = \sz{S'\backslash S} = \sz{S\cup S'}-\sz{S} = \sz{S\cup S'}-\sz{S'} \leq n - (n -f) = f$.

    We assume $f<n/2$. Since $\sz{S\cup S'}\leq n$, then $\sz{S\cap S'} = \sz{S} + \sz{S'} - \sz{S\cup S'} \geq 2(n-f) - n = n - 2f > 0$, indicating $\sz{S\cap S'}\neq \emptyset$. Let $k\in S\cap S'$. By applying the triangle inequality, we arrive at
    \begin{align}
        &\max_{i\in S\backslash S',j\in S' \backslash S} \norm{x_i-x_j} \nonumber\\
        &\leq \max_{i,k\in S}\norm{x_i-x_k} + \max_{i,k\in S'}\norm{x_j-x_k}\,. \label{eqn:detourIneq}
    \end{align}
    Because the {\rm{CenteroW}} algorithm provides a $2$-approximation to the optimum, we attain $r^*\leq r' \leq 2 r^*$.
    
    The Jung's theorem~\cite{Jun01} states that for a set of points  $X$ in a Euclidean space, let $B(c,r)$ be the minimum enclosing ball of $X$,
    \begin{align}
        r \leq \kh{\frac{d}{2d+1}}^{\frac{1}{2}} \cdot \max_{x_i,x_j\in X}\norm{x_i-x_j}\,.\nonumber
    \end{align}
    Then we attain
    \begin{align}
    \label{eqn:diamRadiasBound}
        \sqrt{2} r \leq \max_{x_i,x_j\in X}\norm{x_i-x_j} \leq 2 r\,,
    \end{align}
    in any $d$-dimensional space. The first inequality is attributed to the fact that  $(2d+1)/d > 2$. The second inequality is attained by applying the triangle inequality.

    Then we arrive at
    \begin{align}
        &\max_{i\in S\backslash S',j\in S' \backslash S} \norm{x_i-x_j} \nonumber\\
        &\leq \max_{i,j\in S}\norm{x_i-x_j} + 2 r' \qquad\text{by applying (\ref{eqn:diamRadiasBound})}\nonumber\\
        &\leq \max_{i,j\in S}\norm{x_i-x_j} + 4 r^* \qquad \text{since } r'\leq 2r^*\nonumber\\
        &\leq \max_{i,j\in S}\norm{x_i-x_j} + 4 r_S \qquad\text{optimality of $r^*$}\nonumber\\
        &\leq \max_{i,j\in S}\norm{x_i-x_j} + 2\sqrt{2}\cdot \max_{i,j\in S} \norm{x_i-x_j}\nonumber
    \end{align}
    The last inequality follows by~(\ref{eqn:diamRadiasBound}). Therefore
    \begin{align}
    \label{eqn:CenterFinalBound}
        \max_{i\in S\backslash S',j\in S' \backslash S}\! \norm{x_i-x_j}\! \leq \!(2\sqrt{2}\!+\!1)\! \max_{i,j\in S}\norm{x_i-x_j}.
    \end{align}
     Substituting (\ref{eqn:CenterFinalBound}) into (\ref{eqn:CenterDeviation}) yields the $(f,(2\sqrt{2}+1)f/(n-f))$-resilient averaging guarantee for the CenterwO algorithm.

     Since the CenterwO algorithm is deterministic, 
     by squaring both sides of the result for $(f,\lambda)$-resilient 
     averaging aggregator, we attain the result for $(\delta_{\max},\zeta)$-ARAgg.

     Then we proof the result for $(f,\kappa)$-robustness. We recall the first equality of (40) in~\cite{AGGPS23}
     \begin{align}
         &\kh{1-\frac{\sz{S'\backslash S}}{n-f}}^2\norm{\overline{x}_{S'} - \overline{x}_{S}}^2 = \norm{\frac{1}{n-f}\!\!\sum_{i\in S'\backslash S}\!\! (x_i \! -\! \overline{x}_{S'}) \!-\! \frac{1}{n-f}\!\!\sum_{i\in S\backslash S'}\!\! (x_i \! - \! \overline{x}_{S})}^2\!\!.
         \label{recallMix}
     \end{align}
     From~(\ref{recallMix}) we attain
     \begin{align}
         &\kh{1-\frac{f}{n-f}}^2\norm{\overline{x}_{S'} - \overline{x}_{S}}^2 \nonumber\\
         &\leq \norm{\frac{1}{n-f}\sum_{i\in S'\backslash S} (x_i-\overline{x}_{S'}) - \frac{1}{n-f}\sum_{i\in S\backslash S'} (x_i-\overline{x}_{S})}^2\nonumber\\
         &\leq \frac{1}{(n-f)^2}\kh{\sum_{i\in S'\backslash S} \norm{x_i-\overline{x}_{S'}} +\sum_{i\in S\backslash S'} \norm{x_i-\overline{x}_{S}}}^2\nonumber\\
         &\leq \frac{2f}{(n-f)^2}\kh{\sum_{i\in S'\backslash S} \norm{x_i-\overline{x}_{S'}}^2 +\sum_{i\in S\backslash S'} \norm{x_i-\overline{x}_{S}}^2} \label{eqn:CSineq}
     \end{align}
     The first inequality follows from (40) in~\cite{AGGPS23} and $\sz{S'\backslash S}\leq f$; 
     the second inequality is attained by using a series of triangle inequalities; the third inequality follows by the Cauchy-Schwarz inequality, and the fact that $\sz{S\backslash S'} = \sz{S'\backslash S} \leq f$.

     We observe that
     \begin{align}
         &\sum_{i\in S'\backslash S} \norm{x_i - \overline{x}_{S'}}^2\nonumber\\
         &\leq f \cdot (r')^2 \qquad&\text{by applying~(\ref{eqn:diamRadiasBound})}\nonumber\\
         &\leq 4 f \cdot(r^*)^2 \qquad\qquad&\text{by $r'\leq 2r^*$}\nonumber\\
         &\leq 4 f \cdot(r_{S})^2\qquad\qquad&\text{$r^*$ is optimal}\nonumber\\
         &\leq 4 f \max_{i\in S}\norm{x_i - \overline{x}_S}^2\label{eqn:preintermediate}\\
         &\leq 4f \sum_{i\in S}\norm{x_i - \overline{x}_S}^2\,,
         \label{eqn:intermediate}
     \end{align}
     where the fourth inequality holds because $r_S$ is the radius of the minimum ball enclosing $\{x_i\}_{i\in S}$.
     Therefore $$\kappa = \kh{1+\frac{f}{n-2f}}^2\cdot \frac{8f^2+2f}{n-f}\,,$$
     which concludes the proof for properties of the CenterwO aggregation rule.

    Now we investigate the $(f,\xi)$-robust averaging property of the CenterwO algorithm. We recall (40) in~\cite{AGGPS23}:
    \begin{align}
        & \kh{1-\frac{\sz{S'\backslash S}}{n-f}}^2 \norm{\overline{x}_S' - \overline{x}_S }^2 \leq \frac{2f}{(n-f)^2} \sup_{\norm{v}\leq 1}\sum_{i\in S'\backslash S} \abs{\left\langle v, x_i - \overline{x}_{S'}\right\rangle}^2+ \frac{2f}{(n-f)^2} \sup_{\norm{v}\leq 1}\sum_{i\in S\backslash S'}\abs{\left\langle v, x_i-\overline{x}_{S}\right\rangle}^2\,.
        \label{eqn:robAvgSep}
    \end{align}
    We analyze the supremum of the first term in~(\ref{eqn:robAvgSep}):
    \begin{align}
        & \sup_{\norm{v}\leq 1}\sum_{i\in S'\backslash S} \abs{\left\langle v, x_i - \overline{x}_{S'}\right\rangle}^2 \nonumber\\
        & \leq \sum_{i\in S'\backslash S} \sup_{\norm{v_i}\leq 1}\abs{\left\langle v_i, x_i - \overline{x}_{S'}\right\rangle}^2\nonumber\\
        & = \sum_{i\in S'\backslash S} \norm{x_i- \overline{x}_{S'}}^2\nonumber\\
        &\leq 4f \max_{i\in S}\norm{x_i - \overline{x}_S}^2 \qquad \text{ by applying~(\ref{eqn:preintermediate})} \nonumber\\
        & = 4f \max_{i\in S} \sup_{\norm{v_i}\leq 1} \abs{\left\langle v_i, x_i - \overline{x}_{S}\right\rangle}^2\,,
        \label{eqn:c1}
    \end{align}
    where the two equalities are attained 
    when the vectors $v_i$ and $\overline{x}_{S'}$ (resp. $v_i$ and $\overline{x}_{S}$) have the same direction.  
    We note that for all $i\in S$, there are
    \begin{align}
     \sup_{\norm{v_i}\leq 1} \abs{\left\langle v_i, x_i - \overline{x}_{S}\right\rangle}^2 &= \sup_{\norm{v_i}\leq 1} v_i^\top (x_i - \overline{x}_{S}) (x_i - \overline{x}_{S})^\top v_i \nonumber\\
     &\leq \sup_{\norm{v_i}\leq 1} v_i^\top \kh{M_S} v_i\,.
    \label{eqn:c2}
    \end{align}
    The inequality follows by the Courant-Fischer min-max theorem since a positive semi-definite matrix is added to the rank-one semi-definite matrix in the middle. 
    By combining (\ref{eqn:robAvgSep}), (\ref{eqn:c1}), and (\ref{eqn:c2}) we attain
    \begin{align}
        & \kh{1-\frac{\sz{S'\backslash S}}{n-f}}^2 \norm{\overline{x}_S' - \overline{x}_S }^2
        \leq \frac{8f^2+ 2f}{(n-f)^2} \lambda_{\max}(M_S)\,,\nonumber
    \end{align}
    where the Courant-Fischer min-max theorem is also applied to the second term in (\ref{eqn:robAvgSep}). Note that $\sz{S'\backslash S} \leq f$, we attain the result for $(f,\xi)$-robust averaging of the CenterwO algorithm.  
    
    Next we prove the properties of the MeanwO aggregation rule. Let $\widehat{c}$ be the optimal cluster center for the problem of 1-mean clustering with outliers, and let $\widehat{r}:= \max_{x\in K(\widehat{c}) }\norm{x - \widehat{c}}$. 
    The MeanwO Algorithm finds a cluster center (medoid) $\widetilde{c}$ from the data points. We let $\widetilde{r}:= \max_{x\in K(\widetilde{c}) }\norm{x - \widetilde{c}}$. We further denote by $\widetilde{S}$ the set of vector indices based on which the returned vector is calculated, i.e.\ $\widetilde{S}:=\{i \mid x_i\in K(\widetilde{c})\}$. Similarly, let $\widehat{S}:=\{i \mid x_i\in K(\widehat{c})\}$. 
    


    For the $(f,\kappa)$-robustness, we show that
    \begin{align}
        &\kh{1-\frac{f}{n-f}}^2\norm{\overline{x}_{\widetilde{S}} - \overline{x}_{S}}^2 \nonumber\\
        &\leq  \frac{2f}{(n-f)^2}\kh{\sum_{i\in \widetilde{S}\backslash S} \norm{x_i-\overline{x}_{\widetilde{S}}}^2 +\sum_{i\in S\backslash \widetilde{S}} \norm{x_i-\overline{x}_{S}}^2}\nonumber\\
        &\leq \frac{2f}{(n-f)^2}\kh{\sum_{i\in \widetilde{S}} \norm{x_i-\overline{x}_{\widetilde{S}}}^2 +\sum_{i\in S} \norm{x_i-\overline{x}_{S}}^2}\nonumber\\
        &\leq \frac{2f}{(n-f)^2}\kh{\sum_{i\in \widetilde{S}} \norm{x_i-\widetilde{c}}^2 +\sum_{i\in S} \norm{x_i-\overline{x}_{S}}^2}\nonumber\\
        &\leq \frac{2f}{(n-f)^2}\kh{ 2\cdot \sum_{i\in \widehat{S}} \norm{x_i-\overline{x}_{\widehat{S}}}^2 +\sum_{i\in S} \norm{x_i-\overline{x}_{S}}^2}\nonumber\\
        &\leq \frac{6f}{(n-f)^2} \sum_{i\in S} \norm{x_i-\overline{x}_{S}}^2\,,\nonumber
    \end{align}
    where the second inequality is due to (\ref{eqn:CSineq}); the third inequality follows by the fact that the centroid (or center of mass) $\overline{x}_{\widetilde{S}}$ minimizes the sum of squared distances to the cluster center; the fourth inequality follows by the $2$-approximation guarantee of the MeanwO algorithm; the last inequality is attained since $\widehat{S}$ minimizes $\sum_{i\in S} \norm{x_i - \overline{x}_{S}}$.

    From Propostion 8 and Proposition 9 in~\cite{AFGGPS23} we arrive at the results for $(f,\lambda)$-resilience and $(\delta_{max}, \zeta)$-agnostic robustness of the MeanwO algorithm.

    The result for the $(f,\xi)$-robust averaging property follows straightforwardly from the fact that $\sum_{i\in S} \norm{x_i - \overline{x}_S}^2$ is the trace (sum of eigenvalues) of the matrix $M_S$, which is positive semi-definite and has at most $\min\{n-f, d\}$ non-zero eigenvalues.
\end{proof}

\section{Details for the Two-Phase Algorithm}
\label{appdxSec:twophase}
\begin{algorithm}[tb]
\caption{Outer}
\label{alg:ClusterOuter}
\textbf{Input}: a set of $n$ vectors $X$ in $\mathbb{R}^d$, and an integer $f < \frac{n}{2}$.\\
\textbf{Output}: the mass center of the $n-f$ vectors outside of an approximate minimum 1-center/mean cluster with $f$ in-cluster vectors.
\begin{algorithmic}[1] 
\FOR{ $i = 1,\ldots, n$}
\STATE Find $N(x_i)$, the $f$ closest vectors in $X$ to the vector $x_i$ (including $x_i$), breaking ties arbitrarily;
\STATE Let ${\mathrm{dist}}_i = \mathrm{cost}(x_i, N(x_i))$;\quad
 {\text{/*see (5) for definition of 1-center/mean cost*/}}
\ENDFOR
\STATE Let $j = \argmin_{i} {\mathrm{dist}}_i$;
\RETURN $\frac{1}{n-f}\sum_{x\in (X\backslash N(x_j))} x$;
\end{algorithmic}
\end{algorithm}

Now we describe the Inner and Outer algorithm used in Algorithm~3.
The Inner algorithm should be resilient to siege attacks. In our realization, the Inner algorithm directly calls the 1-Center/Mean Clustering with Outliers algorithm shown in Algorithm~2. The Outer algorithm should be designed to defend against sneak attacks. In our implementation, the Outer algorithm returns the average of the $(n-f)$ vectors \emph{outside} of an approximate minimum 1-center/mean cluster with $f$ in-cluster vectors. The pseudo code of the Outer algorithm is shown in Algorithm~\ref{alg:ClusterOuter}. 

\section{Details for Experiments}
\label{appdxSec:experiments}

\subsection{Simulation Environment}
Both the server and workers are simulated on a cloud virtual machine equipped with a 32-core Intel Xeon Gold 6278@2.6G CPU, 128GB of memory, and a 16GB Quadro RTX 5000 GPU.

\subsection{Attacks}
Here we give the detailed description of adversarial attacks in the experiments.

1) \textbf{LF}: the label flipping attack, where the labels of the images are replaced by labels described by a deterministic permutation in corrupted workers; 

2) \textbf{SF}: the sign flipping attack, where each corrupted worker sends the negative of its true update vector; 

3) \textbf{Gauss}: the Gauss attack, where random Gaussian vectors replace the update vectors with the same vector norm; 

4) \textbf{Omn}: the omniscient attack~\cite{NIPS2017_f4b9ec30}, where all the corrupted workers send the average of all update vectors without corruptions minus the average vector of corrupted workers multiplied by $2n/f$; 

5) \textbf{Empire}: the fall of empire attack ~\cite{XKG20}, where the update vectors or corrupted workers are set to the average of the update vectors without corruptions multiplied by $-0.1$;

6) \textbf{SV}: the scaled variance~\cite{NEURIPS2019_ec1c5914, AELA21} attack, where the corrupted workers set their update vectors to the mean of all workers, shifted by $20$ times the standard deviation in each coordinate.

7) \textbf{PGA}: In particular, we customize more sophisticated attacks, the PGA algorithm, to attack various aggregation rules~\cite{SHKR22}. PGA algorithm leverages STAT-OPT attacks to generate a malicious update, and all Byzantine workers send the same malicious update to the server. As the proposed 2PRASHB algorithm could easily filter out all malicious updates when they are at the same position in multi-dimensional space, we eliminate the data-based stochastic gradient ascent (SGA) in PGA to improve the running efficiency. STAT-OPT computes the average updates $ \nabla^b $ from benign workers, and computes a static malicious direction $ w = -\mathrm{sign}(\nabla^b)$. Moreover, STAT-OPT attacks tailor themselves to the target aggregation rule (Agg) by searching a suboptimal $\gamma$ so that the final malicious update $ \nabla' = \nabla^b - \gamma*w $ could circumvent the target Agg. 

In order to be consistent with the experimental settings of the paper by~\citet{SHKR22}, we only conduct experiments on extremely non-iid dataset drawn from FEMNIST. Corresponding results are presented in Table~\ref{tab:niid_full}, clearly showing the robustness of Cent2P and Mean2P.

\subsection{Architecture of Client Model}
For the image classification task on FEMNIST dataset, we construct a Convolutional Neural Network (CNN) consisting of two convolutional layers. Each convolutional layer has a kernel size of (5 $\times$ 5), and we use 32 and 64 kernels, respectively. After each convolutional layer, we apply a ReLU non-linear activation function followed by a Max-pooling layer with a (2 $\times$ 2) kernel size. We incorporate a fully-connected layer for classification. To train the model, we employ the Cross Entropy loss function and the Stochastic Gradient Descent (SGD) optimizer. 

\subsection{Additional Experimental Evaluations}
Table~\ref{tab:iid_full} and Table~\ref{tab:niid_full} show the full results on FEMNIST dataset for 3 levels of adversarial rates: 0.1, 0.2, and 0.4. Each cell shows the average and standard deviation of testing accuracy in $5$ simulations. The results clearly show the consistent resilience of our methods.

\textbf{CIFAR-10}: To further show the robustness of the proposed aggregation frameworks, we also run experiments on the CIFAR-10 dataset, another typical image classification benchmark. The CIFAR-10 dataset consists of 60000 32x32 color images in 10 classes, with 6000 images per class~\cite{Krizhevsky2009LearningML}. We use a small dataset of 35 clients uniformly sampled from the CIFAR-10 dataset, and each client contains 300 train samples and 60 test samples. As presented in Table~\ref{table:resp_1}, the proposed algorithms show consistent advantages against all baselines.

\begin{table}[!htbp]
    \centering
    \caption{Performance comparison on the uniform sampling (homogeneous) datasets.}
    \label{tab:iid_full}
    \begin{tabular}{ccccccccccccccccccccc}
        \toprule
        Rate & Aggregation & LF & SF & Gauss & Omn & Empire & SV  & Worst\\
        \midrule
        \multirow{8}{*}{0.1} & Avg & 0.77 $\pm$ 0.02 & 0.73 $\pm$ 0.06 & 0.81 $\pm$ 0.01 \cellcolor{lightGreen}& 0.01 $\pm$ 0.00 & 0.81 $\pm$ 0.01 \cellcolor{lightGreen}& 0.75 $\pm$ 0.03 & 0.01\\
        & GM & 0.80 $\pm$ 0.01 \cellcolor{lightGreen}& 0.79 $\pm$ 0.00 \cellcolor{lightGreen}& 0.80 $\pm$ 0.01 & 0.19 $\pm$ 0.27 & 0.80 $\pm$ 0.01 & 0.66 $\pm$ 0.07  & 0.19\\
        & CClip & 0.70 $\pm$ 0.01 & 0.68 $\pm$ 0.01 & 0.70 $\pm$ 0.01 & 0.60 $\pm$ 0.04 & 0.67 $\pm$ 0.01 & 0.68 $\pm$ 0.03 & 0.60\\
        & CWM & 0.60 $\pm$ 0.02 & 0.60 $\pm$ 0.02 & 0.55 $\pm$ 0.06 & 0.32 $\pm$ 0.14 & 0.55 $\pm$ 0.04 & 0.56 $\pm$ 0.03 & 0.32\\
        & CWTM & 0.63 $\pm$ 0.01 & 0.60 $\pm$ 0.03 & 0.66 $\pm$ 0.01 & 0.23 $\pm$ 0.08 & 0.57 $\pm$ 0.02 & 0.28 $\pm$ 0.05 & 0.23\\
        & Krum & 0.51 $\pm$ 0.02 & 0.51 $\pm$ 0.02 & 0.50 $\pm$ 0.02 & 0.52 $\pm$ 0.02 & 0.35 $\pm$ 0.06 & 0.05 $\pm$ 0.00 & 0.05\\
        & Cent2P & 0.78 $\pm$ 0.02 & 0.78 $\pm$ 0.01 & 0.80 $\pm$ 0.01 & 0.80 $\pm$ 0.00 \cellcolor{lightGreen}& 0.80 $\pm$ 0.01 & 0.76 $\pm$ 0.01 \cellcolor{lightGreen}& 0.76\cellcolor{lightGreen}\\
        & Mean2P & 0.78 $\pm$ 0.02 & 0.79 $\pm$ 0.01 \cellcolor{lightGreen}& 0.81 $\pm$ 0.01 \cellcolor{lightGreen}& 0.79 $\pm$ 0.01 & 0.80 $\pm$ 0.01 & 0.76 $\pm$ 0.01 \cellcolor{lightGreen}& 0.76\cellcolor{lightGreen}\\
        \midrule
        \multirow{8}{*}{0.2} & Avg & 0.74 $\pm$ 0.02 & 0.53 $\pm$ 0.25 & 0.80 $\pm$ 0.01\cellcolor{lightGreen} & 0.00 $\pm$ 0.00 & 0.81 $\pm$ 0.00 \cellcolor{lightGreen}& 0.62 $\pm$ 0.02 & 0.00\\
        & GM & 0.78 $\pm$ 0.01 \cellcolor{lightGreen}& 0.76 $\pm$ 0.01 \cellcolor{lightGreen}& 0.80 $\pm$ 0.00 \cellcolor{lightGreen}& 0.05 $\pm$ 0.02 & 0.79 $\pm$ 0.01 & 0.41 $\pm$ 0.09 & 0.05\\
        & CClip & 0.69 $\pm$ 0.01 & 0.64 $\pm$ 0.02 & 0.68 $\pm$ 0.01 & 0.49 $\pm$ 0.05 & 0.57 $\pm$ 0.04 & 0.58 $\pm$ 0.02 & 0.49\\
        & CWM & 0.60 $\pm$ 0.01 & 0.56 $\pm$ 0.03 & 0.59 $\pm$ 0.04 & 0.05 $\pm$ 0.02 & 0.49 $\pm$ 0.06 & 0.44 $\pm$ 0.07 & 0.05\\
        & CWTM & 0.56 $\pm$ 0.03 & 0.55 $\pm$ 0.03 & 0.56 $\pm$ 0.04 & 0.02 $\pm$ 0.01 & 0.38 $\pm$ 0.03 & 0.10 $\pm$ 0.01 & 0.02\\
        & Krum & 0.53 $\pm$ 0.02 & 0.44 $\pm$ 0.05 & 0.49 $\pm$ 0.02 & 0.52 $\pm$ 0.01 & 0.30 $\pm$ 0.05 & 0.05 $\pm$ 0.00 & 0.05\\
        & Cent2P & 0.78 $\pm$ 0.02 \cellcolor{lightGreen}& 0.76 $\pm$ 0.01 \cellcolor{lightGreen}& 0.79 $\pm$ 0.01 & 0.80 $\pm$ 0.01 \cellcolor{lightGreen}& 0.79 $\pm$ 0.00 & 0.73 $\pm$ 0.02 & 0.73\\
        & Mean2P & 0.77 $\pm$ 0.01 & 0.76 $\pm$ 0.02 \cellcolor{lightGreen}& 0.80 $\pm$ 0.01 \cellcolor{lightGreen}& 0.79 $\pm$ 0.01 & 0.79 $\pm$ 0.01 & 0.75 $\pm$ 0.01 \cellcolor{lightGreen}& 0.75\cellcolor{lightGreen}\\
        \midrule
        \multirow{8}{*}{0.4} & Avg & 0.56 $\pm$ 0.03 & 0.04 $\pm$ 0.01 & 0.79 $\pm$ 0.00 \cellcolor{lightGreen}& 0.00 $\pm$ 0.00 & 0.79 $\pm$ 0.01 \cellcolor{lightGreen}& 0.32 $\pm$ 0.04 & 0.00\\
        & GM & 0.64 $\pm$ 0.03 & 0.46 $\pm$ 0.16 & 0.79 $\pm$ 0.00 \cellcolor{lightGreen}& 0.00 $\pm$ 0.00 & 0.74 $\pm$ 0.00 & 0.07 $\pm$ 0.02 & 0.00\\
        & CClip & 0.58 $\pm$ 0.04 & 0.45 $\pm$ 0.06 & 0.64 $\pm$ 0.01 & 0.20 $\pm$ 0.02 & 0.26 $\pm$ 0.04 & 0.26 $\pm$ 0.06 & 0.20\\
        & CWM & 0.50 $\pm$ 0.05 & 0.45 $\pm$ 0.05 & 0.54 $\pm$ 0.04 & 0.02 $\pm$ 0.01 & 0.07 $\pm$ 0.02 & 0.05 $\pm$ 0.00 & 0.02\\
        & CWTM & 0.51 $\pm$ 0.02 & 0.39 $\pm$ 0.03 & 0.55 $\pm$ 0.04 & 0.02 $\pm$ 0.01 & 0.05 $\pm$ 0.01 & 0.05 $\pm$ 0.00 & 0.02\\
        & Krum & 0.53 $\pm$ 0.03 & 0.36 $\pm$ 0.04 & 0.47 $\pm$ 0.02 & 0.10 $\pm$ 0.06 & 0.01 $\pm$ 0.01 & 0.05 $\pm$ 0.00 & 0.01\\
        & Cent2P & 0.74 $\pm$ 0.02 \cellcolor{lightGreen}& 0.62 $\pm$ 0.03 \cellcolor{lightGreen}& 0.76 $\pm$ 0.00 & 0.79 $\pm$ 0.00 \cellcolor{lightGreen}& 0.74 $\pm$ 0.01 & 0.76 $\pm$ 0.02 \cellcolor{lightGreen}& 0.62\cellcolor{lightGreen}\\
        & Mean2P & 0.73 $\pm$ 0.03 & 0.61 $\pm$ 0.01 & 0.76 $\pm$ 0.01 & 0.79 $\pm$ 0.00 \cellcolor{lightGreen}& 0.73 $\pm$ 0.01 & 0.75 $\pm$ 0.02 & 0.61\\
        \bottomrule
    \end{tabular}
\end{table}
\begin{table}[ht]
    \centering
    \setlength{\tabcolsep}{5pt}
    \caption{Performance comparison on the nonuniform sampling (heterogeneous) datasets.}
    \label{tab:niid_full}
    \begin{tabular}{ccccccccccccccccccccc}
        \toprule
        \makebox[0.01\textwidth][c]{Rate} & \makebox[0.08\textwidth][c]{Aggregation} & LF & SF & Gauss & Omn & Empire & SV & PGA & \makebox[0.02\textwidth][c]{Worst}\\
        \midrule
        \multirow{9}{*}{0.1} & Avg & 0.81 $\pm$ 0.01 & 0.30 $\pm$ 0.24 & 0.88 $\pm$ 0.01 \cellcolor{lightGreen}& 0.00 $\pm$ 0.00 & 0.88 $\pm$ 0.01 \cellcolor{lightGreen}& 0.82 $\pm$ 0.01 & 0.03 $\pm$ 0.04 & 0.00\\
        & GM & 0.83 $\pm$ 0.01 & 0.60 $\pm$ 0.22 & 0.88 $\pm$ 0.01 \cellcolor{lightGreen}& 0.09 $\pm$ 0.04 & 0.86 $\pm$ 0.01 & 0.76 $\pm$ 0.02 & 0.00 $\pm$ 0.00 & 0.00\\
        & CClip & 0.72 $\pm$ 0.01 & 0.51 $\pm$ 0.11 & 0.74 $\pm$ 0.01 & 0.50 $\pm$ 0.03 & 0.68 $\pm$ 0.01 & 0.68 $\pm$ 0.01 & 0.43 $\pm$ 0.02 & 0.43\\
        & CWM & 0.58 $\pm$ 0.05 & 0.55 $\pm$ 0.08 & 0.59 $\pm$ 0.06 & 0.11 $\pm$ 0.04 & 0.23 $\pm$ 0.06 & 0.52 $\pm$ 0.03 & 0.12 $\pm$ 0.03 &0.11\\
        & CWTM & 0.61 $\pm$ 0.03 & 0.57 $\pm$ 0.06 & 0.70 $\pm$ 0.02 & 0.11 $\pm$ 0.00 & 0.25 $\pm$ 0.04 & 0.18 $\pm$ 0.05 & 0.04 $\pm$ 0.03 & 0.04\\
        & Krum & 0.29 $\pm$ 0.03 & 0.25 $\pm$ 0.04 & 0.28 $\pm$ 0.02 & 0.30 $\pm$ 0.05 & 0.15 $\pm$ 0.07 & 0.00 $\pm$ 0.00 & 0.19 $\pm$ 0.05 &0.00\\
        & Cent2P & 0.84 $\pm$ 0.01 \cellcolor{lightGreen}& 0.83 \cellcolor{lightGreen}$\pm$ 0.01 & 0.87 $\pm$ 0.01 & 0.88 $\pm$ 0.01 \cellcolor{lightGreen}& 0.87 $\pm$ 0.01 & 0.78 $\pm$ 0.03 & 0.87 $\pm$ 0.00 \cellcolor{lightGreen} & 0.78\\
        & Mean2P & 0.82 $\pm$ 0.02 & 0.83 \cellcolor{lightGreen}$\pm$ 0.02 & 0.87 $\pm$ 0.01 & 0.88 $\pm$ 0.01 \cellcolor{lightGreen}& 0.88 $\pm$ 0.01 \cellcolor{lightGreen}& 0.83 $\pm$ 0.02 \cellcolor{lightGreen} & 0.87 $\pm$ 0.01 \cellcolor{lightGreen} & 0.82 \cellcolor{lightGreen}\\
        \midrule
        \multirow{9}{*}{0.2} & Avg & 0.72 $\pm$ 0.01 & 0.10 $\pm$ 0.02 & 0.87 $\pm$ 0.02 \cellcolor{lightGreen}& 0.00 $\pm$ 0.00 & 0.87 $\pm$ 0.01 \cellcolor{lightGreen}& 0.73 $\pm$ 0.03 & 0.03 $\pm$ 0.04 & 0.00\\
        & GM & 0.76 $\pm$ 0.05 \cellcolor{lightGreen}& 0.10 $\pm$ 0.03 & 0.86 $\pm$ 0.01 & 0.03 $\pm$ 0.02 & 0.84 $\pm$ 0.02 & 0.39 $\pm$ 0.03 & 0.00 $\pm$ 0.00 & 0.00\\
        & CClip & 0.66 $\pm$ 0.04 & 0.25 $\pm$ 0.14 & 0.72 $\pm$ 0.02 & 0.34 $\pm$ 0.02 & 0.49 $\pm$ 0.03 & 0.44 $\pm$ 0.04 & 0.39 $\pm$ 0.01 & 0.25\\
        & CWM & 0.49 $\pm$ 0.03 & 0.50 $\pm$ 0.10 & 0.60 $\pm$ 0.05 & 0.02 $\pm$ 0.02 & 0.04 $\pm$ 0.01 & 0.33 $\pm$ 0.04 & 0.01 $\pm$ 0.01 & 0.01\\
        & CWTM & 0.49 $\pm$ 0.04 & 0.48 $\pm$ 0.06 & 0.63 $\pm$ 0.04 & 0.02 $\pm$ 0.01 & 0.06 $\pm$ 0.02 & 0.04 $\pm$ 0.01 & 0.04 $\pm$ 0.03 & 0.02\\
        & Krum & 0.27 $\pm$ 0.04 & 0.22 $\pm$ 0.03 & 0.26 $\pm$ 0.02 & 0.26 $\pm$ 0.03 & 0.16 $\pm$ 0.07 & 0.01 $\pm$ 0.01 & 0.18 $\pm$ 0.03 & 0.01\\
        & Cent2P & 0.75 $\pm$ 0.07 & 0.69 $\pm$ 0.06 & 0.86 $\pm$ 0.01 & 0.86 $\pm$ 0.02 \cellcolor{lightGreen}& 0.84 $\pm$ 0.02 & 0.78 $\pm$ 0.02 & 0.84 $\pm$ 0.01 & 0.69\\
        & Mean2P & 0.76 $\pm$ 0.05 \cellcolor{lightGreen}& 0.72 $\pm$ 0.04 \cellcolor{lightGreen}& 0.85 $\pm$ 0.01 & 0.85 $\pm$ 0.02 & 0.83 $\pm$ 0.03 & 0.80 $\pm$ 0.01 \cellcolor{lightGreen} & 0.85 $\pm$ 0.01 \cellcolor{lightGreen} & 0.72\cellcolor{lightGreen}\\
        \midrule
        \multirow{9}{*}{0.4} & Avg & 0.57 $\pm$ 0.04 \cellcolor{lightGreen}& 0.07 $\pm$ 0.02 & 0.79 $\pm$ 0.05 \cellcolor{lightGreen}& 0.00 $\pm$ 0.00 & 0.79 $\pm$ 0.05 \cellcolor{lightGreen}& 0.37 $\pm$ 0.03 & 0.03 $\pm$ 0.04 & 0.00\\
        & GM & 0.55 $\pm$ 0.06 & 0.07 $\pm$ 0.03 & 0.78 $\pm$ 0.05 & 0.00 $\pm$ 0.00 & 0.57 $\pm$ 0.08 & 0.06 $\pm$ 0.05 & 0.03 $\pm$ 0.04 & 0.00\\
        & CClip & 0.44 $\pm$ 0.05 & 0.23 $\pm$ 0.03 & 0.63 $\pm$ 0.05 & 0.12 $\pm$ 0.01 & 0.16 $\pm$ 0.02 & 0.19 $\pm$ 0.03 & 0.28 $\pm$ 0.08 & 0.12\\
        & CWM & 0.37 $\pm$ 0.05 & 0.23 $\pm$ 0.05 & 0.57 $\pm$ 0.05 & 0.00 $\pm$ 0.00 & 0.01 $\pm$ 0.02 & 0.02 $\pm$ 0.02 & 0.00 $\pm$ 0.01 & 0.00\\
        & CWTM & 0.37 $\pm$ 0.05 & 0.20 $\pm$ 0.05 & 0.57 $\pm$ 0.04 & 0.02 $\pm$ 0.02 & 0.03 $\pm$ 0.03 & 0.00 $\pm$ 0.00 & 0.05 $\pm$ 0.03 & 0.00\\
        & Krum & 0.14 $\pm$ 0.07 & 0.14 $\pm$ 0.06 & 0.21 $\pm$ 0.04 & 0.13 $\pm$ 0.04 & 0.00 $\pm$ 0.00 & 0.00 $\pm$ 0.00 & 0.08 $\pm$ 0.02 & 0.00\\
        & Cent2P & 0.44 $\pm$ 0.10 & 0.30 $\pm$ 0.26 & 0.74 $\pm$ 0.06 & 0.76 $\pm$ 0.05 \cellcolor{lightGreen}& 0.54 $\pm$ 0.17 & 0.77 $\pm$ 0.05 & 0.77 $\pm$ 0.03 \cellcolor{lightGreen} & 0.30\\
        & Mean2P & 0.43 $\pm$ 0.12 & 0.35 $\pm$ 0.22 \cellcolor{lightGreen}& 0.74 $\pm$ 0.06 & 0.75 $\pm$ 0.04 & 0.55 $\pm$ 0.16 & 0.78 $\pm$ 0.05 \cellcolor{lightGreen} & 0.77 $\pm$ 0.03 \cellcolor{lightGreen} & 0.35\cellcolor{lightGreen}\\
        \bottomrule
    \end{tabular}
\end{table}
\begin{table}[ht]
    \centering
    \Large
    \caption{Performance comparison on CIFAR-10 dataset with the uniform sampling at an adversarial rate of $0.2$.}
    \label{table:resp_1}
    \scalebox{0.7}{
    \begin{tabular}{ccccccccccccccccccccc}
        \toprule
        Aggregation & LF & SF & Gauss & Omn & Empire & SV & Worst \\
        \midrule
        FedAvg & 0.58 & 0.33 & 0.57 & 0.10 & 0.56 & 0.25 & 0.10 \\
        GM & 0.60\cellcolor{lightGreen} & 0.44 & 0.58\cellcolor{lightGreen} & 0.10 & 0.59\cellcolor{lightGreen} & 0.21 & 0.10 \\
        CClip & 0.55 & 0.48 & 0.52 & 0.37 & 0.50 & 0.29 & 0.29 \\
        CWM & 0.47 & 0.39 & 0.46 & 0.09 & 0.35 & 0.16 & 0.09 \\
        CWTM & 0.49 & 0.43 & 0.48 & 0.12 & 0.42 & 0.18 & 0.12 \\
        Krum & 0.18 & 0.15 & 0.21 & 0.21 & 0.10 & 0.10 & 0.10 \\
        Cent2P & 0.59 & 0.49\cellcolor{lightGreen} & 0.56 & 0.57\cellcolor{lightGreen} & 0.57 & 0.46 & 0.46 \\
        Mean2P & 0.59 & 0.47 & 0.56 & 0.56 & 0.57 & 0.49\cellcolor{lightGreen} & 0.47\cellcolor{lightGreen} \\
        \bottomrule
    \end{tabular}
    }
\end{table}

Figure~\ref{fig:iid_std_Result} shows the performance comparison on homogeneous datasets during the training processes. Error bars show the standard deviations. Our methods are among the ones with best performance under all attacks. For homogeneous datasets, CenterwO and MeanwO are the only rules which resist the Omn attack. We have discussed the results for heterogeneous datasets in the main paper. Figure ~\ref{fig:niid_std_Result} shows the results with error bars.

\begin{figure*}[ht]
     \centering
\includegraphics[width=\textwidth]{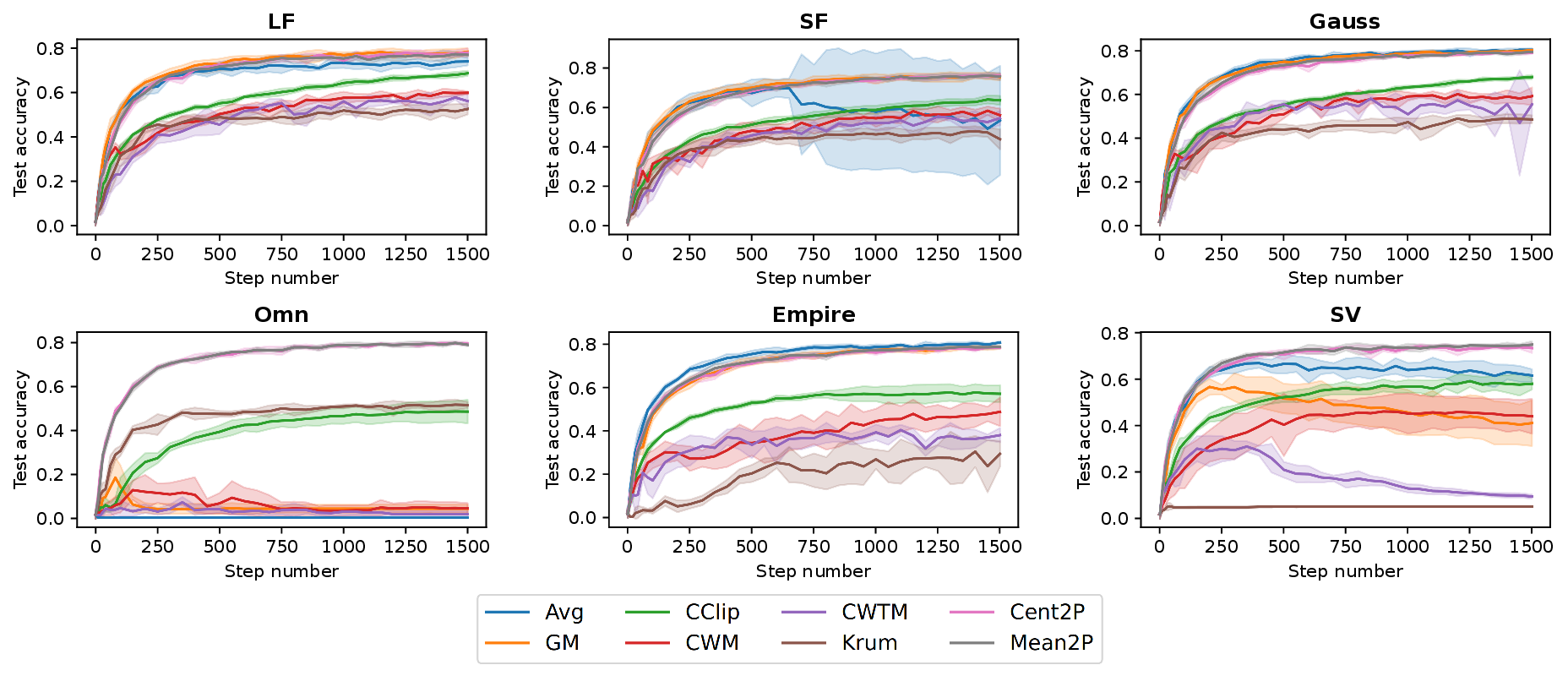}
 \caption{Performance comparison on the homogeneous datasets at an adversarial rate of $0.2$.}
 \label{fig:iid_std_Result}
\end{figure*}

\begin{figure*}[ht]
     \centering
\includegraphics[width=\textwidth]{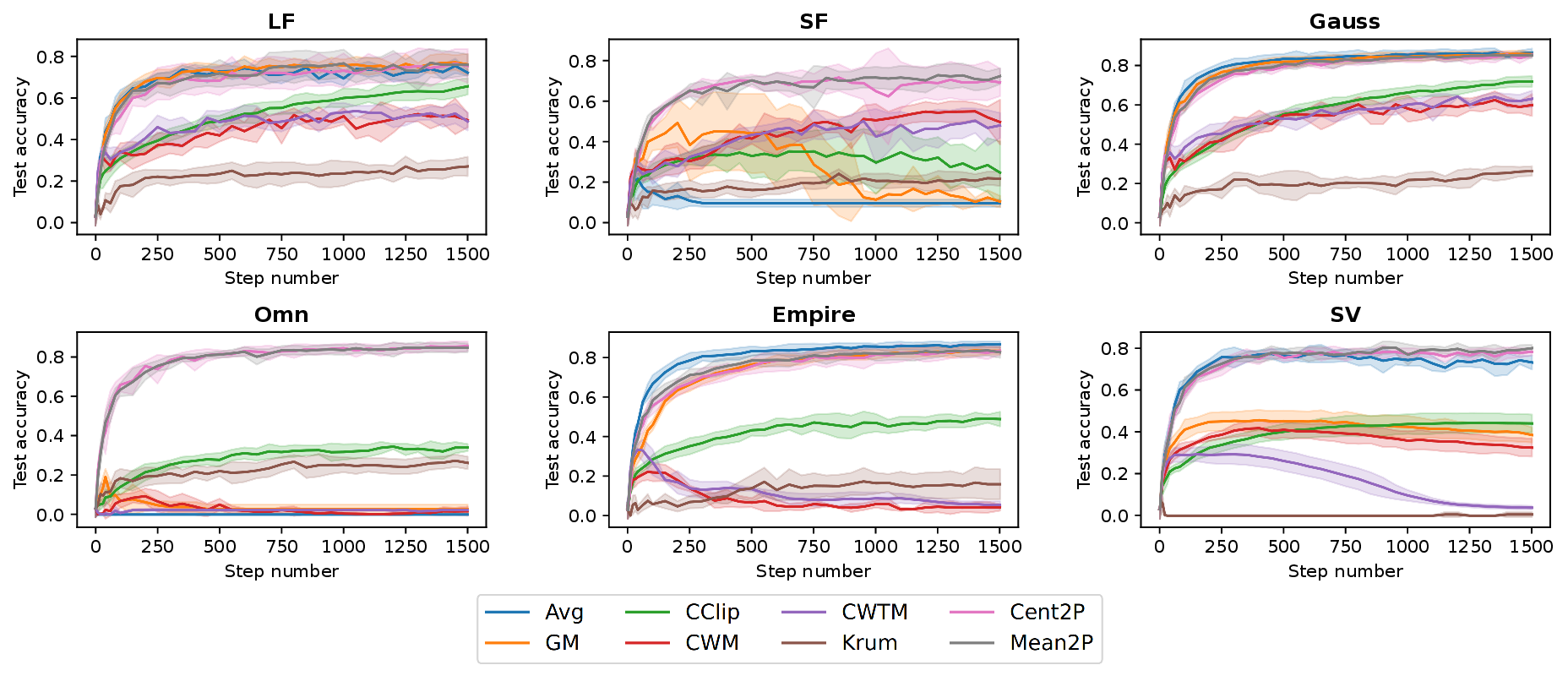}
 \caption{Performance comparison on the heterogeneous datasets at an adversarial rate of $0.2$.}
 \label{fig:niid_std_Result}
\end{figure*}

\subsection{Ablation Experiments}
Table~\ref{tab:ablation} presents a performance comparison between RASHB (Cent1P/Mean1P) and 2PRASHB (Cent2P/Mean2P) on the uniform sampling datasets, with an adversarial rate of $0.2$. Across LF, SF, Gauss, and Empire attacks, both RASHB and 2PRASHB yield comparable outcomes. However, when subjected to Omn and SV attacks, 2PRASHB significantly outperforms PRASHB. Particularly noteworthy is the accuracy achieved under Omn attacks, with Cent2P and Mean2P achieving accuracies of 0.80 and 0.79, respectively, whereas Cent1P and Mean1P only attain 0.12 and 0.01. These results effectively showcase the superiority of 2PRASHB.

\begin{table}[ht]
    \centering
    \Large
    \caption{Performance comparison of RASHB (Cent1P/Mean1P) and 2PRASHB(Cent2P/Mean2P) on the uniform sampling datasets at an adversarial rate of $0.2$.}
    \label{tab:ablation}
    \scalebox{0.78}{
    \begin{tabular}{ccccccccccccccccccccc}
        \toprule
        Aggregation & LF & SF & Gauss & Omn & Empire & SV & Worst \\
        \midrule
        Cent1P & 0.78 $\pm$ 0.02 & 0.76 $\pm$ 0.02 & 0.79 $\pm$ 0.01 & 0.12 $\pm$ 0.14 & 0.75 $\pm$ 0.01 & 0.54 $\pm$ 0.01 & 0.12 \cellcolor{red} \\
        Mean1P & 0.77 $\pm$ 0.02 & 0.76 $\pm$ 0.01 & 0.78 $\pm$ 0.01 & 0.01 $\pm$ 0.00 & 0.77 $\pm$ 0.02 & 0.57 $\pm$ 0.02 & 0.01 \cellcolor{red} \\
        Cent2P & 0.78 $\pm$ 0.02 & 0.76 $\pm$ 0.01 & 0.79 $\pm$ 0.01 & 0.80 $\pm$ 0.01 & 0.79 $\pm$ 0.00 & 0.73 $\pm$ 0.02 & 0.73 \\
        Mean2P & 0.77 $\pm$ 0.01 & 0.76 $\pm$ 0.02 & 0.80 $\pm$ 0.01 & 0.79 $\pm$ 0.01 & 0.79 $\pm$ 0.01 & 0.75 $\pm$ 0.01 & 0.75 \\
        \bottomrule
    \end{tabular}
    }
\end{table}


\end{document}